\documentclass[twoside]{article}

\usepackage[accepted]{aistats2020}
\usepackage{footmisc}
\usepackage{amsmath,amsthm}
\usepackage{xcolor}
\usepackage{graphicx}

\usepackage[utf8]{inputenc}
\usepackage[T1]{fontenc}
\usepackage{lmodern}

\usepackage{algorithm}
\usepackage{algpseudocode}
\usepackage{booktabs}
\usepackage{amsfonts}
\usepackage{amssymb}
\usepackage{pifont}
\usepackage{multirow}
\usepackage{siunitx}
\usepackage{enumitem}
\usepackage{hyperref}

\newtheorem{lemma}{Lemma}
\makeatletter
\newcommand\footnoteref[1]{\protected@xdef\@thefnmark{\ref{#1}}\@footnotemark}
\makeatother
\usepackage[round]{natbib}

\newcommand{\xmark}{\ding{53}}%
\newcommand\blfootnote[1]{%
  \begingroup
  \renewcommand\thefootnote{}\footnote{#1}%
  \addtocounter{footnote}{-1}%
  \endgroup
}

\begin{document}
\setlength{\abovedisplayskip}{3pt}
\setlength{\belowdisplayskip}{3pt}
%

%
\runningauthor{Avinava Dubey, Michael Zhang, Eric P. Xing, Sinead A. Williamson}

\twocolumn[

\aistatstitle{Distributed, partially collapsed MCMC for Bayesian nonparametrics}

\aistatsauthor{ Avinava Dubey$^{\ast}$ \And Michael M. Zhang$^{\ast}$ \And  Eric P. Xing \And  Sinead A. Williamson}

\aistatsaddress{ Carnegie Mellon Univ. \And  Princeton University  \And Carnegie Mellon Univ. \And Univ. of Texas Austin }


]

\begin{abstract}
 Bayesian nonparametric (BNP) models provide elegant methods for discovering underlying latent features within a data set, but inference in such models can be slow. We exploit the fact that completely random measures, which commonly-used models like the Dirichlet process and the beta-Bernoulli process can be expressed using, are decomposable into independent sub-measures.   We use this decomposition to partition the latent measure into a finite measure containing only instantiated components, and an infinite measure containing all other components. We then select different inference algorithms for the two components: uncollapsed samplers mix well on the finite measure, while collapsed samplers mix well on the infinite, sparsely occupied tail. The resulting hybrid algorithm can be applied to a wide class of models, and can be easily distributed to allow scalable inference without sacrificing asymptotic convergence guarantees. 
 \blfootnote{$*$ denotes equal contribution}
\end{abstract}

\section{\uppercase{Introduction}}

\begin{table*}[t]
\small
\vspace{-2mm}
\newcommand{\ra}[1]{\renewcommand{\arraystretch}{#1}}
\ra{1.1}
    \centering
    \caption{Comparison of various parallel and distributed inference algorithm proposed for BNPs\label{tab:dist_comparisions}. Other CRMs include gamma-Poisson process, beta-negative binomial process etc.}
    \vspace{1mm}
    \begin{tabular}{@{} l c c c c c c c c c @{}} 
    \toprule
    \multirow{3}{*}{\small \textsf{Methods}} &
    {\small \textsf{Data}} &
    \multirow{3}{*}{\small \textsf{Exact}} & \multirow{3}{*}{\small \textsf{Parallel}} & \multirow{3}{*}{\small \textsf{Distributed}} & \multicolumn{5}{c}{\small BN Processes} \\ \cmidrule{6-10}
    & {\small \textsf{Size}}& & & & {\tiny Beta-Bernoulli} & {\tiny Other} &{\tiny DP} & {\tiny HDP} & {\tiny Pitman-Yor } \\
     & Millions & & & &{\tiny  Process} & {\tiny  CRMs} & & & {\tiny Process} \\
    \midrule
    {\tiny \citealt{Smyth:Welling:Asuncion:2009}} & $1$M & \xmark & \checkmark & \checkmark & \xmark & \xmark & \checkmark & \checkmark & \xmark\\
    {\tiny \citealt{Doshi-Velez:Ghahramani:2009}} & $.01$M & \checkmark & \xmark & \xmark & \checkmark & \xmark & \xmark & \xmark & \xmark\\
    {\tiny \citealt{Doshi-Velez:2009}} & $.1$M & \xmark & \checkmark & \checkmark & \checkmark & \xmark & \xmark & \xmark & \xmark\\
    {\tiny \citealt{Williamson:Dubey:Xing:2013}} & $1$M & \checkmark & \checkmark & \xmark & \xmark & \xmark & \checkmark & \checkmark & \xmark\\
    {\tiny \citealt{chang2013parallel}}  & $0.3$M & \checkmark & \checkmark & \xmark & \xmark & \xmark & \checkmark & \checkmark & \xmark\\
    {\tiny \citealt{Dubey:Williamson:Xing:2014}} & $10$M & \checkmark & \checkmark & \xmark & \xmark & \xmark & \xmark & \xmark & \checkmark\\
    {\tiny \citealt{Ge:Chen:Wan:Ghahramani:2015}} & $0.1$M & \checkmark & \checkmark & \checkmark & \xmark & \xmark & \checkmark & \checkmark & \xmark\\
    {\tiny \citealt{Yerebakan:Dundar:2017}} & $0.05$M & \checkmark & \checkmark & \checkmark & \xmark & \xmark & \checkmark & \xmark & \xmark\\
    \midrule
    {This paper} & $1$M & \checkmark & \checkmark & \checkmark & \checkmark & \checkmark & \checkmark & \checkmark & \checkmark \\ 
    \bottomrule
    \end{tabular}
\end{table*}



Bayesian nonparametric (BNP) models are a flexible class of models whose complexity adapts to the data under consideration. BNP models place priors on infinite-dimensional objects, such as partitions with infinitely many blocks; matrices with infinitely many columns; or discrete measures with infinitely many atoms. A finite set of observations is assumed to be generated from a finite---but random---subset of these components, allowing flexibility in the underlying dimensionality and providing the ability to incorporate previously unseen properties as our dataset grows.

While the flexibility 
of these models is a good fit for large, complex data sets, distributing existing inference algorithms across multiple machines is challenging. If we explicitly represent subsets of the underlying infinite-dimensional object---for example, using a slice sampler---we can face high memory requirements and slow convergence. Conversely, if we integrate out the infinite-dimensional object, we run into problems due to induced global dependencies.

Moreover, a key goal of distributed algorithms is to minimize communication between agents. This can be achieved by breaking the algorithm into independent sub-algorithms, which can be run independently on different agents. In practice, we usually cannot split an MCMC sampler on a Bayesian hierarchical model into entirely independent sub-algorithms since there are typically some global dependencies implied by the hierarchy. Instead, we make use of conditional independencies to temporarily partition our algorithm.

\textbf{Contributions:} In this paper, we propose a distributable sampler for models derived from completely random measures, which unifies exact parallel inference for a wide class of Bayesian nonparametric priors, including the popularly used Dirichlet process \citep{Ferguson:1973} and the beta-Bernoulli process \citep{Griffiths:Ghahramani:2011}. 
After reviewing the appropriate background material, we first introduce general recipes for (non-distributed) partially collapsed samplers appropriate for a wide range of BNP models, focusing on the beta-Bernoulli process and the Dirichlet process as exemplars. We then demonstrate that these methods can be easily extended to a distributed setting. Finally we provide experimental results for our hybrid and distributed sampler on DP and BB inference. 

\section{\uppercase{ Related Work}}\label{sec:bnp}



Completely random measures \citep[CRMs,][]{Kingman:1967} are random measures that assign independent masses to disjoint subsets of a space. For example, the gamma process assigns a gamma-distributed mass to each subset. Other examples include the beta process \citep{Hjort:1990} and the Poisson process. The distribution of a CRM is completely determined by its L\'{e}vy measure, which controls the size and location of atoms.

Many nonparametric distributions can be expressed in terms of CRMs. For example, if we sample $B = \sum_{i=1}^\infty \mu_i \delta_{\theta_i}$ from a (homogeneous) beta process, and generate a sequence of subsets $Z_i$ where $\theta_k\in Z_i$ w.p.\ $\mu_k$, then we obtain an exchangeable distribution over sequences of subsets known as the beta-Bernoulli process \citep{Thibaux:Jordan:2007}, which is related to the Indian buffet process \citep[IBP,][]{Griffiths:Ghahramani:2005}.  If we sample $G$ from a gamma process on $\Omega$ with base measure $\alpha H$, then $D(\cdot)=G(\cdot)/G(\Omega)$ is distributed according to a Dirichlet process with concentration parameter $\alpha$ and base measure $H$. 


Inference in such models tend to fall into three categories: uncollapsed samplers that alternate between sampling the latent measure and the assignments \citep{Ishwaran:Zarepour:2002,Paisley:Carin:2009,Zhou:2009,Walker:2007,Teh:Gorur:Ghahramani:2007}; collapsed samplers where the latent measure is integrated out \citep{Ishwaran:James:2001,Neal:1998,Griffiths:Ghahramani:2005}; and optimization-based methods that work with approximating distributions where the parameters are assumed to have a mean-field distribution \citep{Blei:Jordan:2006,Doshi:Miller:2009}. 

Collapsed methods often mix slowly due to the dependency between assignments, while blocked updates mean uncollapsed methods typically have good mixing properties at convergence \citep{Ishwaran:James:2001}. Uncollapsed methods are often slow to incorporate new components, since they typically rely on sampling unoccupied components from the prior. In high dimensions, such components are unlikely to be close to the data. Conversely, collapsed methods can make use of the data when introducing new points, which tends to lead to faster convergence \citep{Neal:1998}. 

Other methods incorporate both uncollapsed and collapsed sampling, resulting in a ``hybrid'', partially collapsed sampler, although such approaches have been restricted to specific models. \citet{Doshi-Velez:Ghahramani:2009} introduced a linear time accelerated Gibbs sampler for conjugate IBPs that effectively marginalized over the latent factors, while more recently \cite{Yerebakan:Dundar:2017} developed a sampler by partially marginalizing latent random measure for DPs. These methods can be seen as special cases of our hybrid framework (Section~\ref{sec:hybrid}), but do not generalize to the distributed setting.


Several inference algorithms allow computation to be distributed across multiple machines---although again, such algorithms are specific to a single model. The approximate collapsed algorithm of  \citet{Smyth:Welling:Asuncion:2009} is only developed for Dirichlet process-based models, and lacks asymptotic convergence guarantees. Distributed split-merge methods have been developed for Dirichlet process-based models, but not extended to more general nonparametric models \citep{chang2013parallel,chang2014parallel}. Partition-based algorithms based on properties of CRMs have been developed for Dirichlet process- and Pitman-Yor process-based models \citep{Williamson:Dubey:Xing:2013,Dubey:Williamson:Xing:2014}, but it is unclear how to extend to other model families. A low-communication, distributed-memory slice sampler has been developed for the Dirichlet process, but since it is based on an uncollapsed method it will tend to perform poorly in high dimensions \citep{Ge:Chen:Wan:Ghahramani:2015}. \citet{Doshi-Velez:2009} developed an approximate distributed inference algorithm for the Indian buffet process which is superficially similar to our distributed beta-Bernoulli sampler. However, their approach allows all processors to add new features, which will lead to overestimating the number of features. 
We contrast these methods
in Table~\ref{tab:dist_comparisions}.

\section{\uppercase{Hybrid Inference for CRM-based models}}\label{sec:hybrid}
By definition, completely random measures can be decomposed into independent random measures. If the CRM has been transformed in some manner we can often still decompose the resulting random measure into independent random measures -- for example, a normalized random measure can be decomposed as a mixture of normalized random measures. Such representations allow us to decompose our inference algorithms, and use different inference techniques on the constituent measures.

As discussed in Section~\ref{sec:bnp}, collapsed and uncollapsed methods both have advantages and disadvantages. Loosely, collapsed methods are good at adding new components and exploring the tail of the distribution, while uncollapsed methods offer better mixing in established clusters and easy parallelization. We make use of the decomposition properties of CRMs to partition our model into two components: One containing (finitely many) components currently associated with multiple observations, and one containing the infinite tail of components. 

\subsection{Models Constructed Directly From CRMs}

Consider a generic hierarchical model,
\begin{equation}
\begin{aligned}
& M:=\sum_{k=1}^\infty \mu_k\delta_{\theta_k}\sim \mbox{CRM}(\nu(d\mu)H(d\theta))\\
& Z_{i,k} \sim f(\mu_k), \quad \quad
X_i = \sum_{k=1}^\infty g(Z_{i,k}, \theta_k) + \epsilon_i
\end{aligned}\label{eqn:CRM_prior}
\end{equation}
where $\nu$ is a measure on $\mathbb{R}_+$; $H$ is a measure on the space of parameters $\Theta$; $g(\cdot, \cdot)$ is some deterministic transformation such that $g(0, \theta)=0$ for all $\theta\in \Theta$; $\epsilon_i$ is a noise term; and $f(\cdot)$ is a likelihood that forms a conjugate pair with $M$, i.e.\ the posterior distribution $P(M^*|Z)$ is a CRM in the same family. The indices $i=1, \ldots , n$ refer to the observations and the indices $k=1, 2, \ldots $ refer to the features. This framework includes exchangeable feature allocation models such as the beta-Bernoulli process \citep{Griffiths:Ghahramani:2005,Thibaux:Jordan:2007}, the infinite gamma-Poisson feature model \citep{Titsias:2008}, and the beta negative binomial process \citep{Zhou:Hannah:Dunson:Carin:2012,Broderick:Mackey:Paisley:Jordan:2014}. 
We assume, as is the case in these examples, that both collapsed and uncollapsed posterior inference algorithms can be described. We also assume for simplicity that the prior contains no fixed-location atoms, although this assumption could be relaxed \citep[see][]{Broderick:Wilson:Jordan:2018}.

\begin{lemma}[\citealt{Broderick:Wilson:Jordan:2018}]
If $M \sim \mbox{CRM}(\nu(d\mu)H(d\theta))$ and $Z_{i,k}\sim f(\mu_k)$ for $i=1,\dots, n$, and if $M$ and $f$ form a conjugate pair, then the posterior $P(M^*|\{Z_i\})$ can be decomposed into two CRMs, each with known distribution. If $K$ is the number of features for which $\sum_i Z_{i,k}>0$, the first, $M^*_{\leq K} = \sum_{k=1}^K\mu^*_k\delta_{\theta_k}$, is a finite measure with fixed-location atoms at locations $\theta_k: \sum_i Z_{i,k}>0$. The distribution over the corresponding weights is proportional to $\nu(d\mu)\prod_{i=1}^n f(Z_{i, k}|\mu)$. The second, with infinitely many random-location atoms, has L\'{e}vy measure $\nu(d\mu)H(d\theta)\left(f(0|\mu)\right)^n$. \label{lem:CRM}
\end{lemma}
Based on Lemma~\ref{lem:CRM}, we partition $M^*$ into a finite CRM $M^*_{\leq J} =  \sum_{k=1}^J\mu^*_k\delta_{\theta_k}$ for some $J\leq K$, that contains all, or a subset of, the fixed-location atoms; and an infinite CRM $M^*_{>J}$ that contains the remaining atoms. We use an uncollapsed sampler to sample $M^*_{\leq J}|\{Z_{i,k}\}_{k\leq K}$, and then sample $\{Z_{i,k}\}_{k\leq K}|X, M^*_{\mbox{\small{fixed}}} 
$. Then, we use a collapsed sampler to sample the allocations $Z_{i,k}: k>J$.
The size $J$ should be changed periodically to make sure $J\leq K$ to avoid explicitly instantiating atoms that are not associated with data. In our experiments, we set $J=K$ at the beginning of each iteration. 

\begin{algorithm}[t]
\caption{Hybrid Beta-Bernoulli Sampler\label{algo:hybrid_ibpmm}}
\begin{algorithmic}[1]
\While{not converged}
\State Select $J$
\State Sample $\mu_k \sim \mbox{Beta}(m_k, n - m_k + c)$, $\forall \ k\leq J$
\State Sample $ \theta_k \sim p(\theta_k|H, Z, X)$, $\forall \ k\leq J$
\For{$i=1, \dots, N$}
\State Sample $\{Z_{i,k}\}_{i=1, k \leq J}^n$ according to Eq. \ref{ibp_finite}
\State Sample $\{Z_{i,k}\}_{i=1, k>J}^n$ according to Eq. \ref{ibp_mk}
\State Metropolis-Hastings sample $Z_{i,k^{\prime}}$ for  
$k^{\prime} \in \{ k :  Z_{i,k}=1,  \sum_{j \neq i} Z_{j,k}=0 \}$.
\EndFor
\EndWhile
\end{algorithmic}
\end{algorithm}

\textbf{Example 1: Beta-Bernoulli Process.}
As a specific example, consider the beta-Bernoulli process. Let $B:=\sum_{k=1}^\infty \mu_k\delta_{\theta_k} \sim \mbox{BetaP}(\alpha, c, H)$
be a homogeneous beta process \citep{Thibaux:Jordan:2007}, and let 
$Z_{i,k}\sim\mbox{Bernoulli}(\mu_k)$. The posterior is given by 
$$B^*|Z \sim \mbox{BetaP}\left(\frac{c\alpha + \sum_k m_k}{c+n}, c+n, \frac{c\alpha H + \sum_k m_k\delta_{\theta_k}}{c\alpha + \sum_k m_k}\right)$$
where $m_k = \sum_{i=1}^N Z_{i,k}$. In this case the following lemma helps us in decomposing the posterior distribution:
\begin{lemma}[\citealt{Thibaux:Jordan:2007}]
If $K$ is the number of features where $\sum_i Z_{i,k}>0$ and $J\leq K$, we can decompose the posterior distribution of beta-Bernoulli process as as $B^* = B^*_{\leq J} + B^*_{>J}$ where
{\small
\begin{equation}
  \begin{aligned}
    B^*_{\leq J} \sim& \mbox{BetaP}\left(\frac{\sum_{k=1}^J m_k}{c+n},c+n,\frac{\sum_k m_k\delta_{\theta_k}}{\sum_{k=1}^J m_k}\right)\\
    B^*_{>J} \sim& \mbox{BetaP}\left(\frac{c\alpha}{c+n},c+n, \frac{c\alpha H + \sum_{k>J}m_k\delta_{\theta_k}}{c\alpha + \sum_{k>J}m_k}\right)\label{eqn:hybrid_ibp}
  \end{aligned}
\end{equation}}
\end{lemma}

We note that the atom sizes of $B^*_{\leq J}$ are $\text{Beta}(m_k, n-m_k+c)$ random variables. This allows us to split the beta-Bernoulli process into two independent feature selection mechanisms: one with a finite number of currently instantiated features, and one with an unbounded number of features.

 The likelihood of a given data point, given the latent variables and other data points (written as $P(X_i|-)$ for space reasons), can be written  as
\begin{equation}
\begin{aligned}
    p(X_i|-) = \int &p(X_i|Z_i, \theta_1, \dots, \theta_J, \dots, \theta_K)\\ &p(\theta_{J+1},\dots, \theta_{K}|Z_{-i}, X_{-i})d\theta_{J+1}\cdots d\theta_K.\label{eqn:lhood}
    \end{aligned}
\end{equation}
When working with conjugate likelihoods, we can typically evaluate the integral term in Equation~\ref{eqn:lhood} analytically (see \citet{Griffiths:Ghahramani:2005} for the case of Gaussian prior and linear Gaussian likelihood for $J=0$). If this is not possible, we can sample $\theta_k$  for $J \geq k\geq K$ as auxiliary variables \citep{Doshi-Velez:Ghahramani:2009}.

This formulation of the likelihood, combined with the partitioning of the Bernoulli process described in Equation~\ref{eqn:hybrid_ibp}, gives us the hybrid sampler, which we summarize in Algorithm \ref{algo:hybrid_ibpmm}. For each data point $X_i$, we sample $Z_{i,k}$ in a three step manner. For $k \leq J$,
    \begin{equation}
\begin{aligned}
 &P(Z_{i,k} = z|-) \propto \\
 & \begin{cases}\mu_k p(X_i|Z_{i,k}=1,-) & z=1\\ (1-\mu_k)  p(X_i|Z_{i,k}=0,-) & z=0.
  \end{cases} \label{ibp_finite} 
\end{aligned}
\end{equation}
where $p(X_i|Z_{i,k}=1,-)$ is given by Equation~\ref{eqn:lhood} with $Z_{i,k}$ set to 1. For $k>J$ and $m_k>0$, we have
    \begin{equation}
    \begin{aligned}
  &P(Z_{i,k} = z|Z_{-(i,k)},X_i) \propto \\
  & \begin{cases}m_k p(X_i|Z_{i, k}=1, -) & z=1\\
    (n-m_k) p(X_i|Z_{i,k}=0, -) & z=0.
  \end{cases}\label{ibp_mk}  \end{aligned}
\end{equation}
where $p(X_i|Z_{i,k}=1,-)$ is given by Equation~\ref{eqn:lhood} with $Z_{i,k}$ set to 0. 
Finally, we propose adding $\mbox{Poisson}(\alpha/n)$ new features, accepting using a Metropolis-Hastings step.
Once we have sampled the $Z_{i,k}$, for every instantiated feature $k\leq J$, we sample $\mu_k \sim \mbox{Beta}(m_k, n - m_k + c)$ and its corresponding parameters $\theta_k\sim p(\theta_k|H,Z,X)$. 

We note that similar algorithms can be easily derived for other nonparmetric latent feature models such as those based on the infinite gamma-Poisson process \citep{Thibaux:Jordan:2007} and the beta-negative Binomial process \citep{Zhou:Hannah:Dunson:Carin:2012,Broderick:Wilson:Jordan:2018}.

\subsection{Models Based on Transformations of Random Measures}

While applying transformations to CRMs means the posterior is no longer directly decomposable, in some cases we can still apply the above general ideas. 

\textbf{Example 2: Dirichlet Process.}
As noted in Section~\ref{sec:bnp}, the Dirichlet process with concentration parameter $\alpha$ and base measure $H$ can be constructed by normalizing a gamma process with base measure $\alpha H$. If the Dirichlet process is used as the prior in a mixture model (DPMM), the posterior distribution conditioned on the cluster allocations $Z_1, \dots, Z_n$, having $K$ unique clusters is again a Dirichlet process: 
\begin{equation}
    D^*| Z_1, \ldots Z_n \sim \mbox{DP}(\alpha + n, \frac{\alpha H + \sum_{k<K} m_k \delta_{\theta_k}}{n + \alpha})
\end{equation}
where $m_k = \sum_i \delta(Z_i, k)$ and $K$ is the number of clusters with $m_k > 0$. In this case also the following lemma helps us in decomposing the posterior. 
\begin{lemma}
Assuming $J\leq K$, and $\tilde{n} = \sum_{k\leq J} m_k$, we can decompose the posterior of DP as $$  D^*|Z_1, \dots, Z_n = B^* D^*_{\leq J} + (1-B^*)D^*_{> J}$$ where
{\small
\begin{equation}
    \begin{aligned}
\nonumber    D^*_{\leq J}|Z_1, \dots, Z_n \sim& \mbox{DP}(\tilde{n}, \frac{\sum_{k\leq J}m_k\delta_{\theta_k}}{\tilde{n}})\\
    D^*_{> J}|Z_1, \dots, Z_n \sim& \mbox{DP}(\alpha + n - \tilde{n}, \frac{\alpha H + \sum_{k> J}m_k\delta_{\theta_K}}{\alpha + n - \tilde{n}})\\
     B^* \sim& \mbox{Beta}(\tilde{n}, n-\tilde{n}+\alpha) 
    \end{aligned}
\end{equation}}\label{lem:DP}
\end{lemma}
\begin{proof}
This is a direct extension of the fact that the Dirichlet process has Dirichlet-distributed marginals \citep{Ferguson:1973}. See Chapter 3 of \citet{Ghosh:Ramamoorthi:2003} for a detailed analysis.
\end{proof}
\begin{algorithm}[t]
\caption{Hybrid DPMM Sampler\label{algo:hybrid_dpmm}}
\begin{algorithmic}[1]
\While{not converged}
\State Select $J$
\State Sample $B^* \sim \mbox{Beta}(\tilde{n}, n-\tilde{n} + \alpha)$
\State Sample $(\pi_1,\dots,\pi_J) \sim \mbox{Dir}(m_1, \ldots, m_J)$
\State For $k\leq J$, sample 
 $\theta_k \sim p(\theta_k|H, \{X_i:Z_i=k\})$
\State Sample $Z_i$, $\forall i = 1,\ldots, n$, using Equation \ref{eqn:special}
\EndWhile
\end{algorithmic}
\end{algorithm}
We note that the posterior atom weights $(\pi_1,\dots, \pi_J)$ for the finite component are distributed according to $\mbox{Dirichlet}(m_1,\dots, m_J)$, and can easily be sampled as part of an uncollapsed sampler.  Conditioned on $\{\pi_k, \theta_k: k\leq J\}$ and $B^*$ we can sample the cluster allocation, $Z_i$ of point $X_i$ as
\begin{equation}P(Z_i=k|-) \propto \begin{cases} B^*\pi_k f(x_n;\theta_k) & k\leq J\\
    \frac{(1-B^*)m_k}{\sum_{j>K} m_j + \alpha}f_k(x_n) & J < k \leq K\\
    \frac{(1-B^*)\alpha}{\sum_{j>K} m_j + \alpha}f_H(x_n) & k=K+1
  \end{cases}
  \label{eqn:special}
\end{equation}
where $f(X_i;\theta_k)$ is the likelihood for each mixing component; $f_k(X_i) = \int_\Theta f(X_i;\theta)p(\theta|\{X_j: Z_j=k, j\neq i\})d\theta$ is the conditional probability of $x_i$ given other members of the $k$th cluster; and $f_H(x_i) = \int_\Theta f(x_i;\theta) H(d\theta)$. This procedure is summarized in Algorithm~\ref{algo:hybrid_dpmm}.

\textbf{Example 3: Pitman-Yor processes}
The Pitman-Yor process \citep{perman1992size,pitman1997two} is a distribution over probability measures, parameterized by $0\leq\sigma<1$ and $\alpha>-\sigma$, that is obtained from a $\sigma$‐stable CRM via a change of measure and normalization. Provided $\alpha>0$, it can also be represented as a Dirichlet process mixture of normalized $\sigma$-stable CRMs \citep[Lemma 22,][]{pitman1997two}. This representation allows us to decompose the posterior distribution into a beta mixture of a finite-dimensional Pitman-Yor process and an infinite-dimensional Pitman-Yor process. We provide more details in the supplementary section \ref{sec:pymm_supli}.

\textbf{Example 4: Hierarchical Dirichlet processes} We can decompose the hierarchical Dirichlet process \citep[HDP,][]{Teh:Jordan:Beal:Blei:2006} in a manner comparable to the Dirichlet process, allowing our hybrid sampler to be used on the HDP. For space reasons, we defer discussion to the supplementary section \ref{sec:hdp_supli}.

\section{\uppercase{Distributed Inference for CRM-based models}}\label{method}

\begin{algorithm}[t]
\caption{Distributed Beta Bernoulli Sampler\label{algo:distributed_ibp}}
\begin{algorithmic}[1]
\Procedure{Local}{$\{X_i, Z_i\}, P$} 

\Comment{Global variables $J$, $P^*$, $\{\theta_k, \mu_k\}_{k=1}^J, n$}
\For{$k \leq J$}
\State Sample $\{Z_{i, k}\}$ according to \eqref{ibp_finite}
\EndFor
\If{$P=P^*$}
\State For $k>J$, sample $Z_{i,k}$ according to \eqref{ibp_mk}
\State Sample $\mbox{Poisson}(\alpha/n)$ new features
\EndIf
\EndProcedure

\Procedure{Global}{$\{X_i, Z_i\}$}
\State Gather feature counts $m_k$ and parameter sufficient statistics $\Psi_k$ from all processors.
\State Let $J$ be the number of instantiated features.
\State For $k: m_k > 0$, sample 
$$\begin{aligned}\mu_k \sim& \mbox{Beta}(m_k, n - m_k + c)\\
\theta_k \sim& p(\theta_k|\Psi_k, H)\end{aligned}$$
\State Sample $P^* \sim \mbox{Uniform}(1, \ldots, P)$
\EndProcedure
\end{algorithmic}
\end{algorithm}
\vspace{-1em}
The sampling algorithms in Section~\ref{sec:hybrid} can easily be adapted to a distributed setting, where data are partitioned across several $P$ machines and communication is limited. In this setting, we set $J=K$ (i.e.\ the number of currently instantiated features) after every communication step. We instantiate the finite measure ($M^*_{\leq J}$ in the case of CRMs, $D^*_{>J}$ for DPMMs), with globally shared atom sizes and locations, on all processors. 

We then randomly select one out of $P$ processors by sampling $P^*\sim \mbox{Uniform}(1,\dots, P)$. On the remaining $P-1$ processors, we sample the allocations $Z_i$ using restricted Gibbs sampling \citep{Neal:1998}, enforcing $Z_{i, k}=0$ for $k>J$. When working with DPMMs, this means that we only need to calculate cluster probabilities that depend on the instantiated measure $D^*_{\leq J}$. In the CRM case, it means that we avoid the integral in Equation~\ref{eqn:lhood}, and hence avoid any dependence on $Z_{-i}$ or $X_{-i}$ when conditioning on $M^*_{\leq J}$. In both cases, this means that we do not need knowledge of the feature allocations on other processors, and can sample the $Z_i$ independent of each other.

On the remaining processor $P^*$, we sample the $Z_i$ using unrestricted Gibbs sampling. Let $\mathcal{P}^*$ be the set of indices of data points on processor $P^*$. Since we know that $Z_{j, k}=0$ for all $j\not\in \mathcal{P}^*$, then we can calculate the full sufficient statistics for features $k>J$ without knowledge of data or latent features on other processors. While $P(Z_i|-)$ does depend on $\{Z_j, X_j: j\in \mathcal{P}\}$, it is independent of  $\{Z_j, X_j: j\not \in \mathcal{P}\}$ conditioned on $D^*_{\leq J}$ (or $M^*_{\leq J}$) plus the fact that $Z_{j, k}=0$ for all $j\not\in \mathcal{P}^*$, so can be calculated without further information from the other processors.

At each global step, we gather the sufficient statistics from all instantiated clusters -- from both the finite component $M^*_{\leq J}$ / $D^*_{\leq J}$ and the infinite component $M^*_{>J}$ / $D^*_{>J}$ -- and sample parameters for those clusters. We then create a new partition, redefining $J$ as the current number of instantiated component parameters. In the case of the DPMM, we also resample $B\sim \mbox{Beta}(N,\alpha)$.
We summarize the distributed algorithm for the special cases of the beta-Bernoulli process and the DPMM in Algorithms \ref{algo:distributed_ibp} and \ref{algo:distributed_dpmm} and for PYMM in algorithm \ref{algo:distributed_pymm} in supplementary.


\begin{algorithm}[t]
\caption{Distributed DPMM Sampler\label{algo:distributed_dpmm}}
\begin{algorithmic}[1]
\Procedure{Local}{$\{X_i, Z_i\}$} 

\Comment{Global variables $J, P^*,\{\theta_k, \pi_k\}_{k=1}^J, B^*$}
\If {$P = P^{\ast}$}
\State Sample $Z_i$ according to \eqref{eqn:special}
\Else
\State $P(Z_i = k) \propto \pi_k f(X_i;\theta_k)$
\EndIf
\EndProcedure

\Procedure{Global}{$\{X_i, Z_i\}$}
\State Gather cluster counts $m_k$ and parameter sufficient statistics $\Psi_k$ from all processors.
\State Sample $B^* \sim \mbox{Beta}(n, \alpha)$
\State Let $J$ be the number of instantiated clusters.
\State Sample $(\pi_1, \dots, \pi_J) \sim \mbox{Dir}(m_1, \ldots, m_J)$
\State For $k: m_k>0$, sample 
$ \theta_k \sim p(\theta_k|\Psi_k, H).$
\State Sample $P^* \sim \mbox{Uniform}(1, \ldots, P)$
\EndProcedure
\end{algorithmic}
\end{algorithm}

\subsection{Warm-start Heuristics}\label{sec:warm}
In our distributed approach, only $1/P$ of the data points eligible to start a new cluster or feature, meaning that the rate at which new clusters/features are added will decrease with the number of processors. This can lead to slow convergence if we start with too few clusters. To avoid this problem, we initialize our algorithm by allowing \textit{all} processors to instantiate new clusters. At each global step, we decrease the number of randomly selected processors eligible to instantiate new clusters, until we end up with a single processor. This approach tend to over-estimate the number of clusters. However, the procedure acts in a manner similar to simulated annealing, by encouraging large moves early in the algorithm but gradually decreasing the excess stochasticity until we are sampling from the correct algorithm.
We note that a sampler with multiple processors instantiating new clusters is not a correct sampler until we revert to a single processor proposing new features, because correct MCMC samplers are invariant to its starting position. 

\section{\uppercase{Experimental evaluation}}\label{experiments}
While our primary contribution is in the development of distributed algorithms, we first consider, in Section~\ref{sec:hybrid_experiments}, the performance of the hybrid algorithms developed in Section~\ref{sec:hybrid} in a non-distributed setting. We show that this performance extends to the distributed setting, and offers impressive scaling, in Section~\ref{sec:dist_exp}.

\subsection{Evaluating the Hybrid Sampler}\label{sec:hybrid_experiments}
We begin by considering the performance of the hybrid samplers introduced in Section~\ref{sec:hybrid} in a non-distributed setting. For this, we focus on the Dirichlet process, since there exist a number of collapsed and uncollapsed inference algorithms; we expect similar results under other models.

\begin{figure}[t]
\begin{center}
\includegraphics[width=\linewidth]{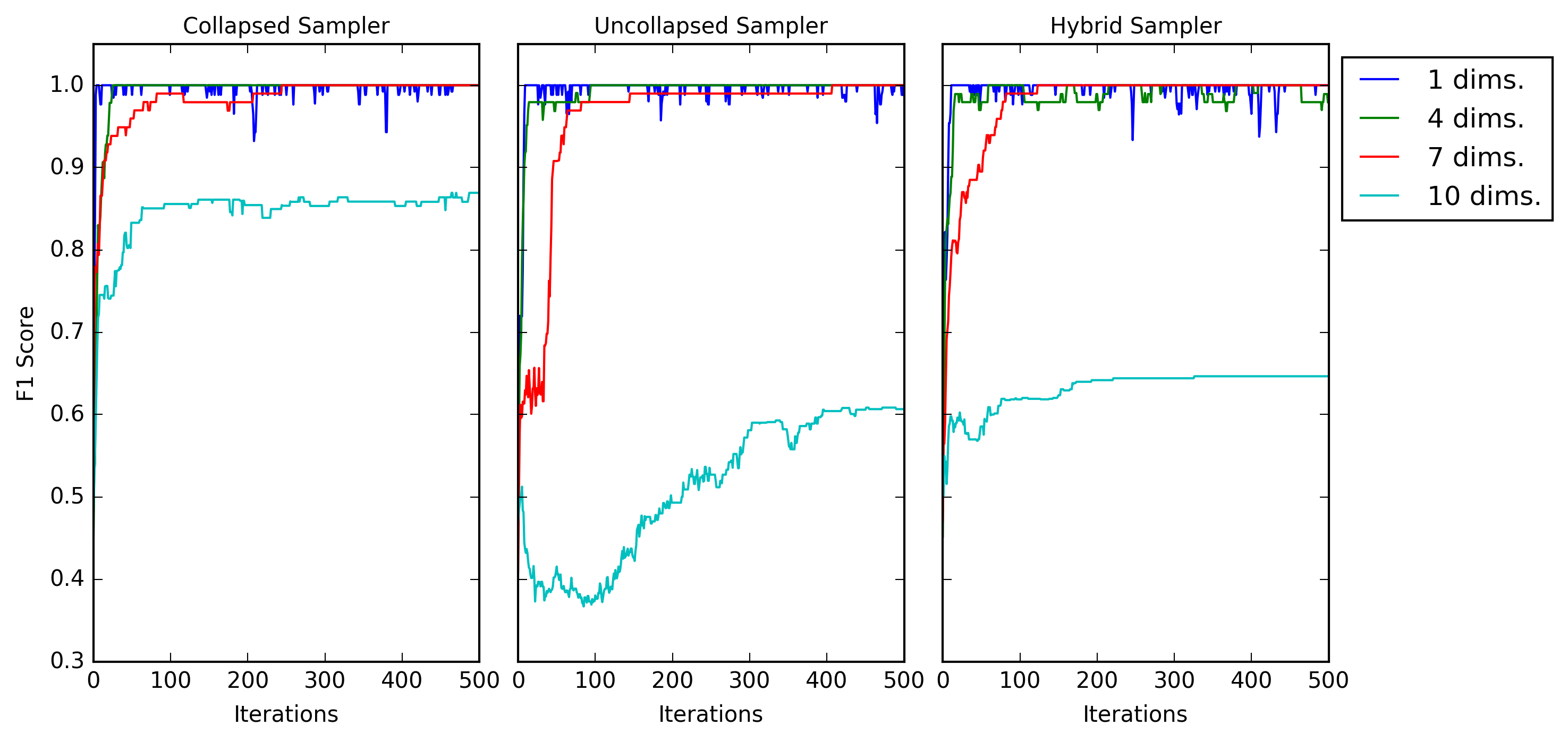}\label{fig:dp_comparison}

\caption{Synthetic data experiments. Comparison of F1 scores over iteration for the collapsed, uncollapsed and hybrid samplers for growing dimensions. }\label{fig:synthetic_comparison}
\end{center}
\end{figure}

\begin{figure*}[h]

\begin{center}
\includegraphics[width=.24\linewidth]{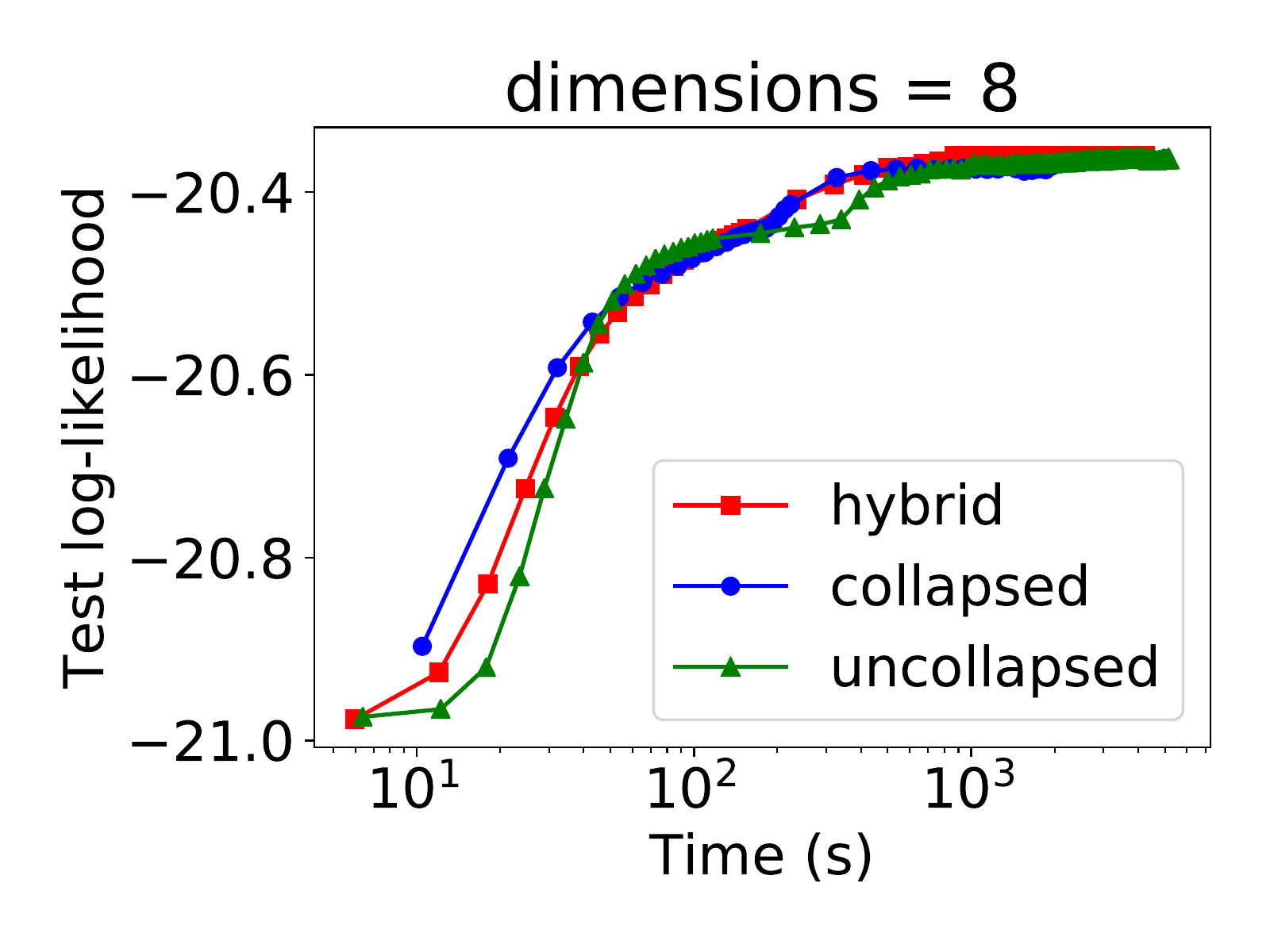}
\includegraphics[width=.24\linewidth]{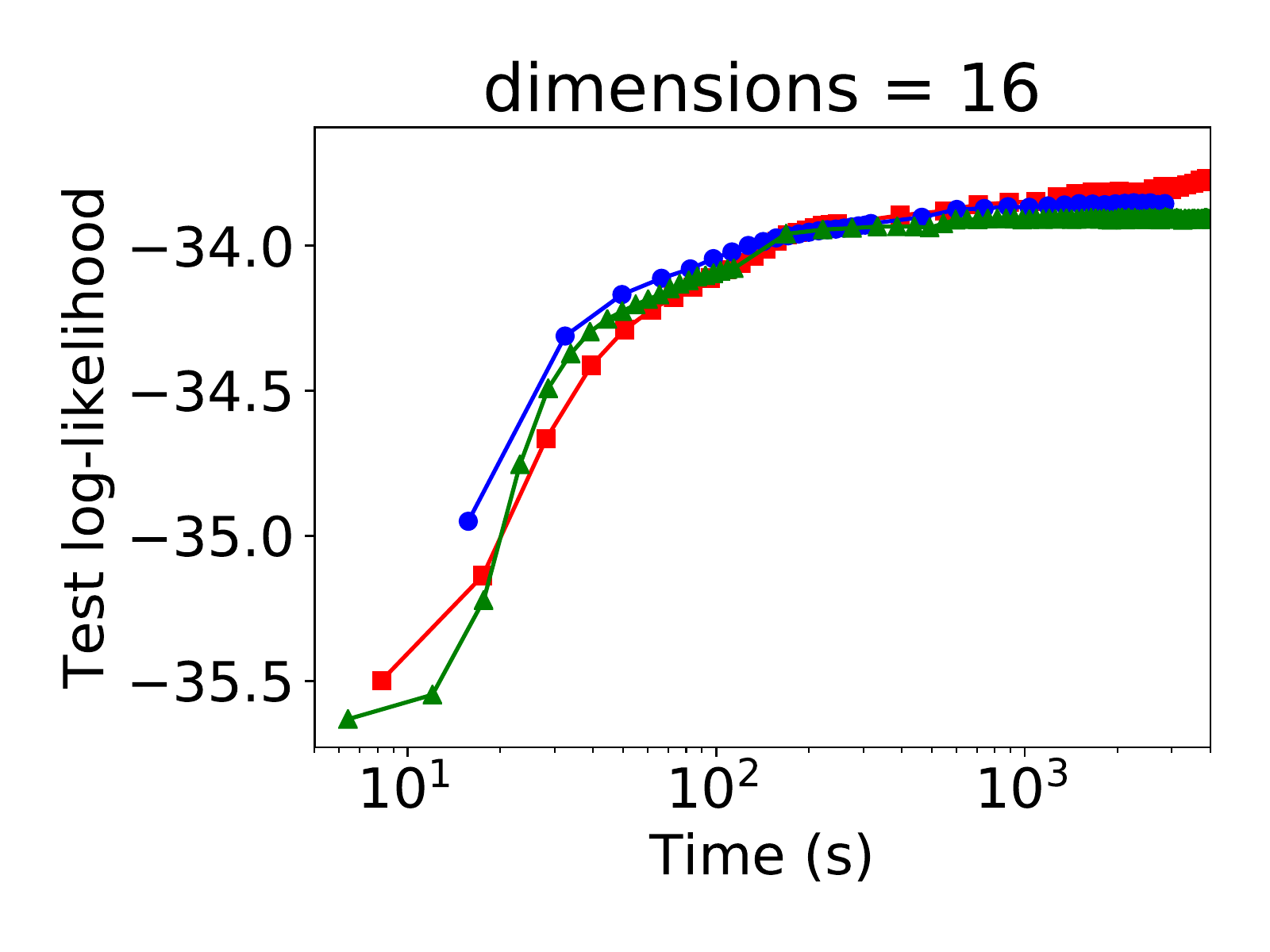}
\includegraphics[width=.24\linewidth]{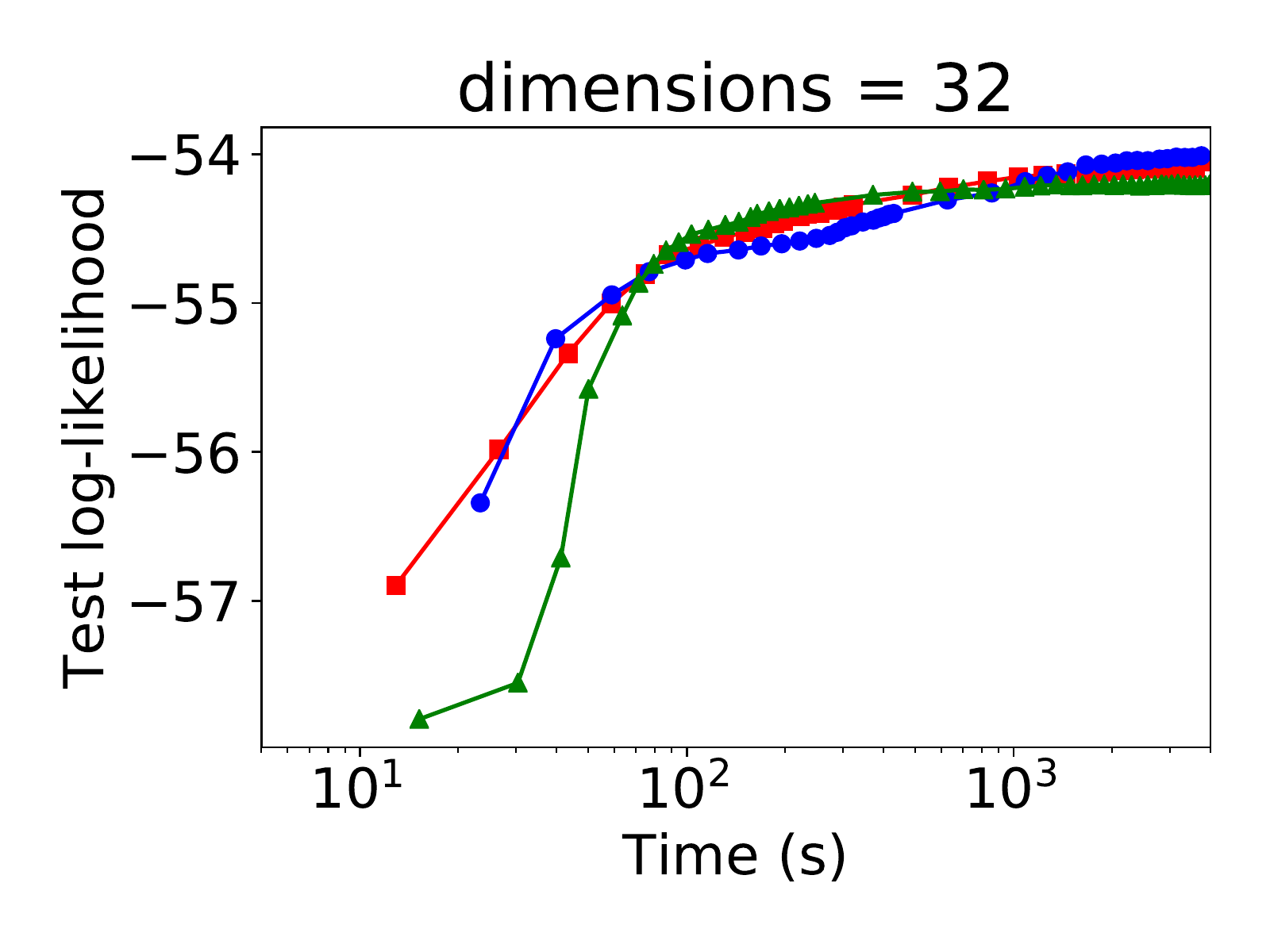}
\includegraphics[width=.24\linewidth]{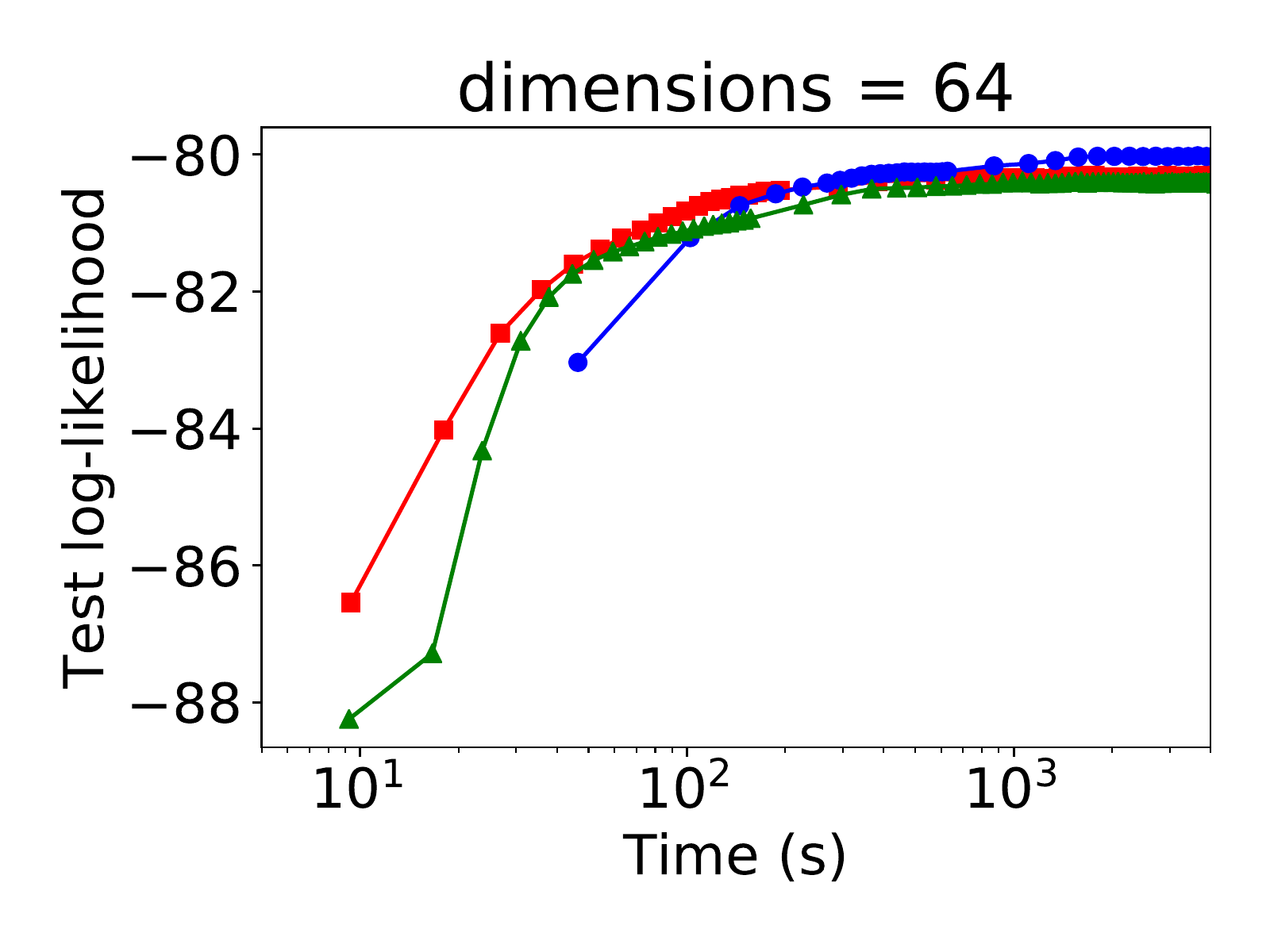}
\caption{Test log-likelihood with time on CIFAR-100 dataset as the dimensionality increases from 8 to 64.\label{fig:cfar_dp_dim}}
\end{center}
\end{figure*}

We compare the hybrid sampler of Algorithm~\ref{algo:hybrid_dpmm} with a standard collapsed Gibbs sampler and an uncollapsed sampler based on Algorithm 8 of \citet{Neal:1998}. Algorithm 8 collapses occupied clusters and instantiates a subset of unoccupied clusters; we modify this to instantiate the atoms associated with unoccupied clusters. Concretely, at each iteration, we sample weights for the $K$ instantiated clusters plus $U$ uninstantiated clusters as $(\pi_1, \dots, \pi_K, \pi_{K+1}, \dots, \pi_{K+U}) \sim \mbox{Dir}\left(m_1, \dots, m_K, \frac{\alpha}{J},\dots, \frac{\alpha}{J}\right)$, 
and sample locations for the uninstantiated clusters from the base measure $H$. 
Note that this method can be distributed 
easily.

Figure~\ref{fig:synthetic_comparison} shows convergence plots for the three algorithms. The data set is a $D$ dimensional synthetic data set consisting of $100$ observations of Gaussian mixtures with $2$ true mixture components centered at $5\times\left\{ 1 \right\}^{D}$ and $-5\times\left\{ 1 \right\}^{D}$ with an identity covariance matrix.

While the three algorithms perform comparably on low-dimensional data, as the dimension increases the performance of the uncollapsed sampler degrades much more than the collapsed sampler. This is because in high dimensions, it is unlikely that a proposed parameter will be near our data, so the associated likelihood of any given data point will be low. This is in contrast to the collapsed setting, where we integrate over all possible locations. While the hybrid method performs worse in high dimensions than the collapsed method, it outperforms the uncollapsed method.

The synthetic data in Figure~\ref{fig:synthetic_comparison} has fairly low-dimensional structure, so we do not see negative effects due to the poor mixing of the collapsed sampler. Next, we evaluate the algorithms on the CIFAR-100 dataset \citep{Krizhevsky:2009}. We used PCA to reduce the dimension of the data to between 8 and 64, and plot the test set log likelihood over time in Figure~\ref{fig:cfar_dp_dim}. Each marker represents a single iteration. We see that the uncollapsed sampler requires more iterations to converge than the collapsed sampler; however since each iteration takes less time, in some cases the wall time to convergence is slower. The hybrid method has comparable iteration time to the collapsed, but, in general, converges faster. We see that, even without taking advantage of parallelization, the hybrid method is a compelling competitor to pure-collapsed and pure-uncollapsed algorithms.

\subsection{Evaluating the Distributed Samplers}\label{sec:dist_exp}

Here, we show that the distributed inference algorithms introduced in Section~\ref{method} allow inference in BNP models to be scaled to large datasets, without sacrificing accuracy. We focus on two cases: the beta-Bernoulli process (Algorithm~\ref{algo:distributed_ibp}) and the Dirichlet process (Algorithm~\ref{algo:distributed_dpmm})\footnote{Our code is available online at \href{https://github.com/michaelzhang01/hybridBNP}{\texttt{https://github.com/michaelzhang01/hybridBNP}}}.

\subsubsection{Beta-Bernoulli Process}
\begin{figure*}[h]
\centering
\includegraphics[width=0.325\linewidth]{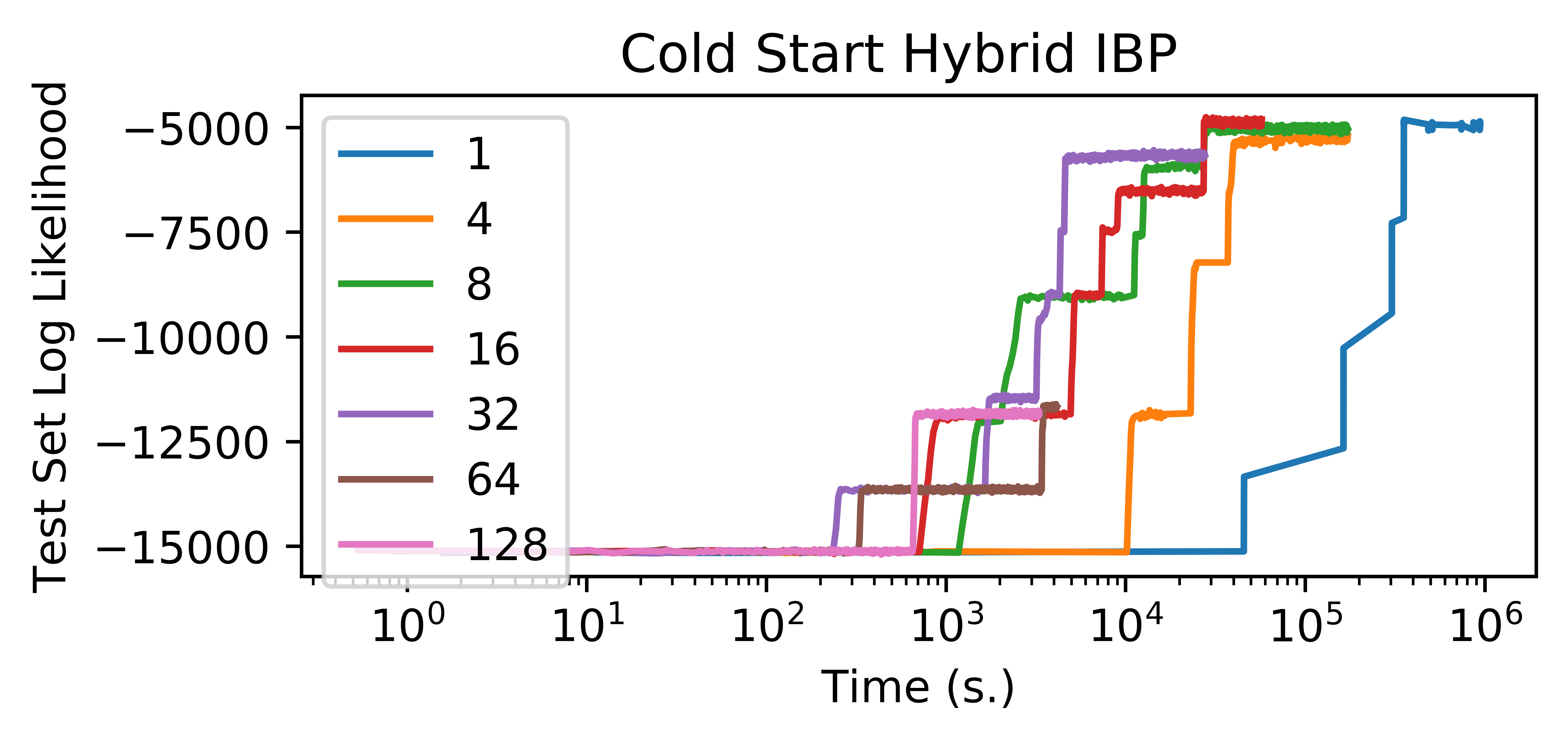}\label{fig:LL_cold}
\includegraphics[width=0.325\linewidth]{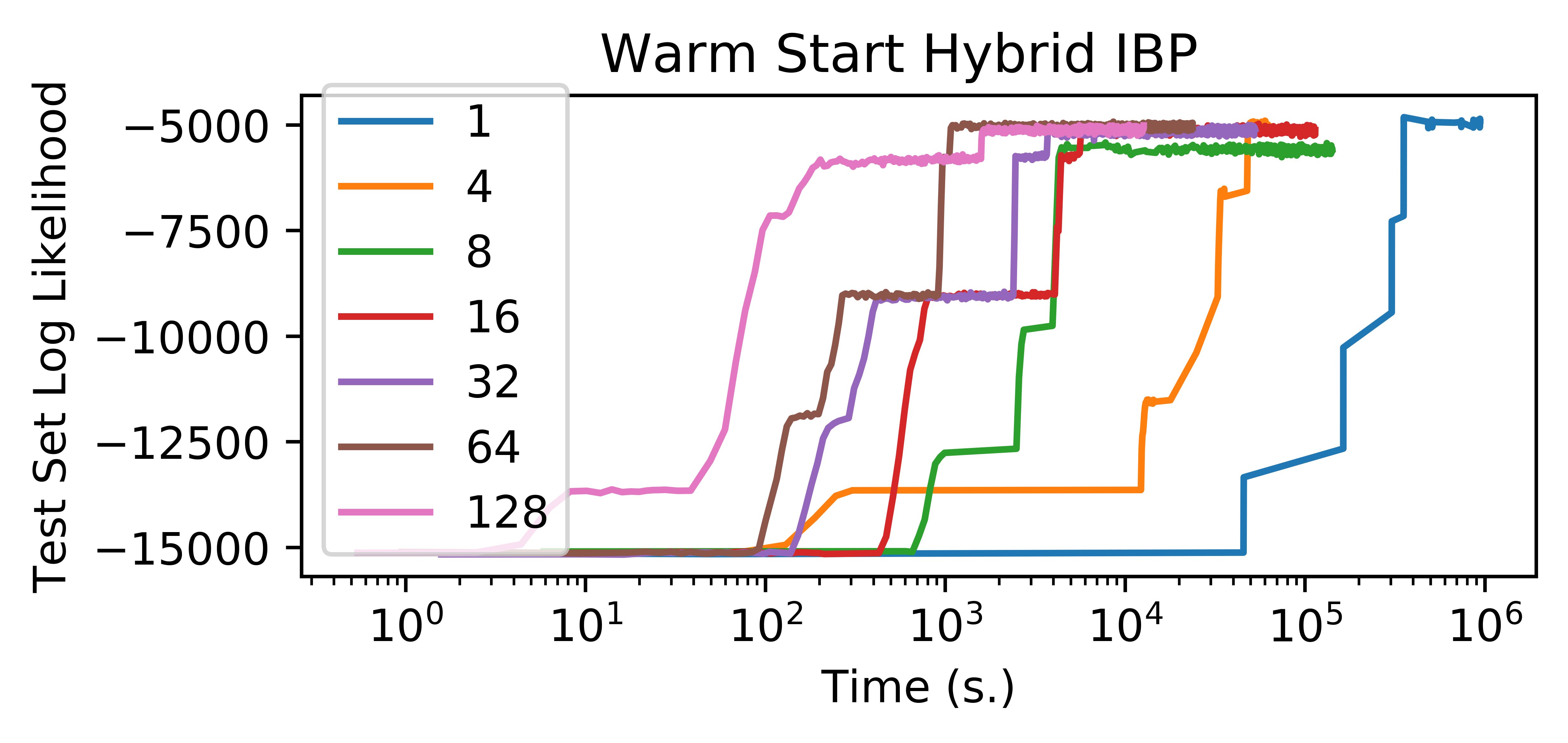}\label{fig:LL_warm}
\includegraphics[width=0.325\linewidth]{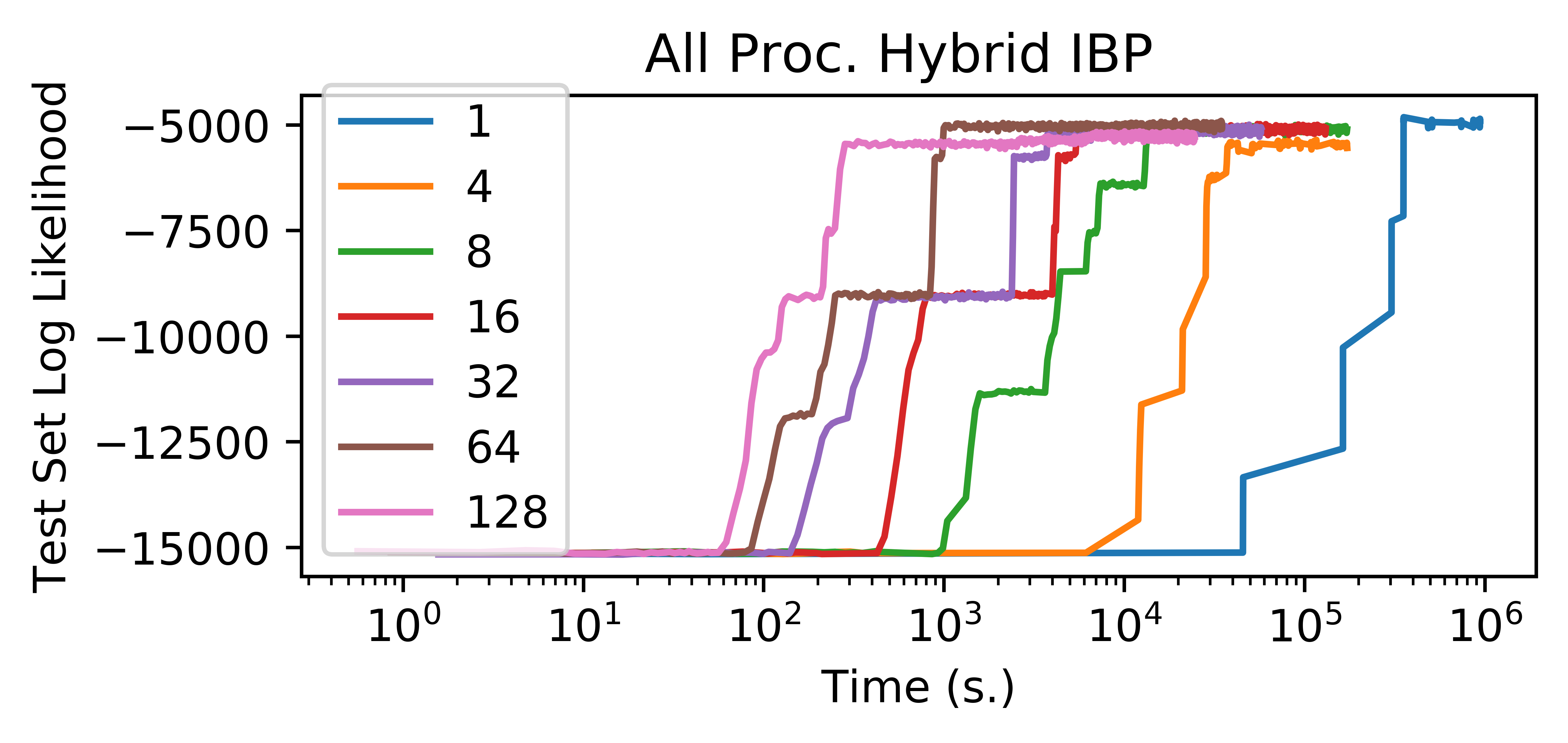}\\
\includegraphics[width=0.325\linewidth]{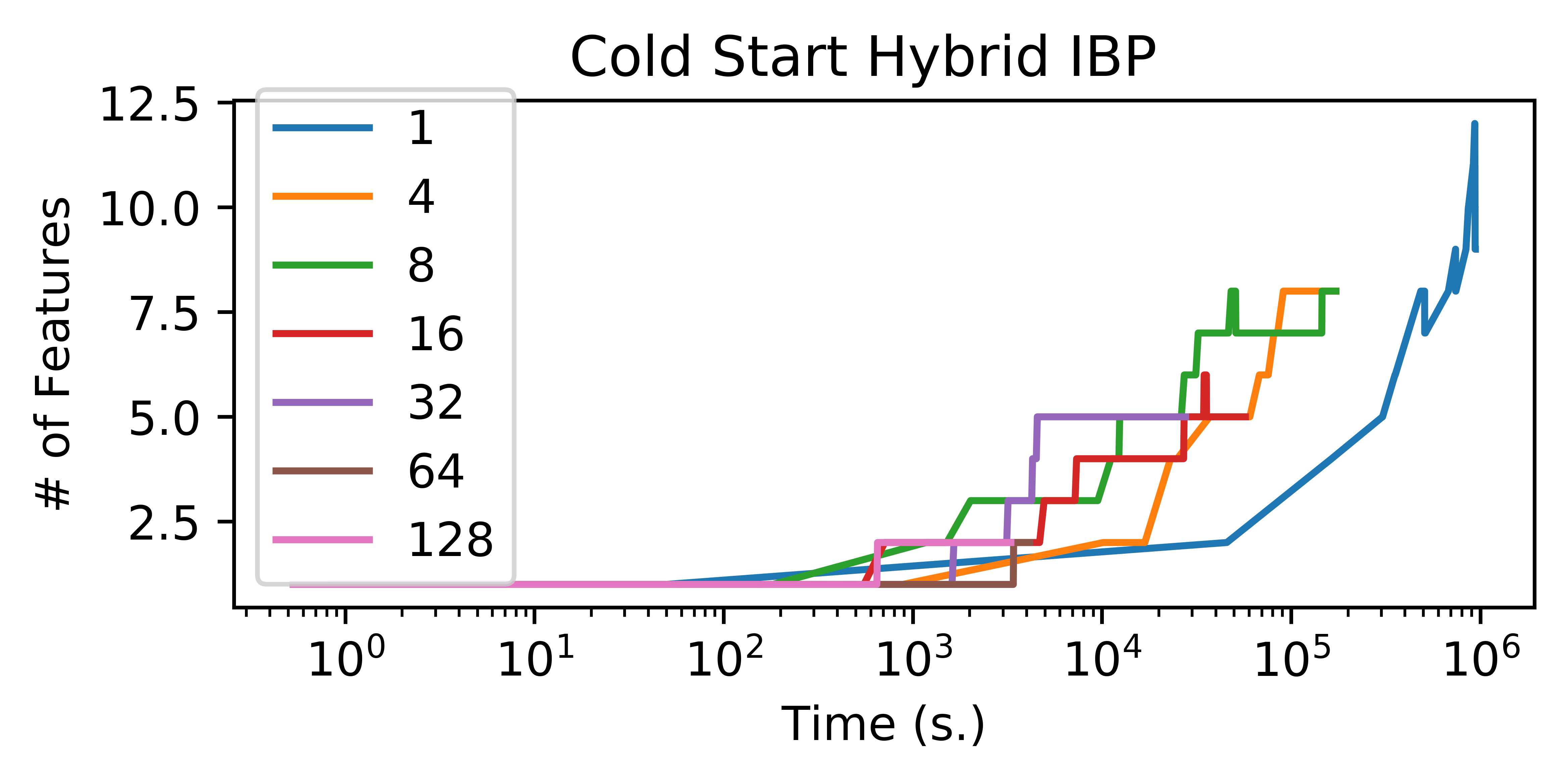}\label{fig:feat_cold}
\includegraphics[width=0.325\linewidth]{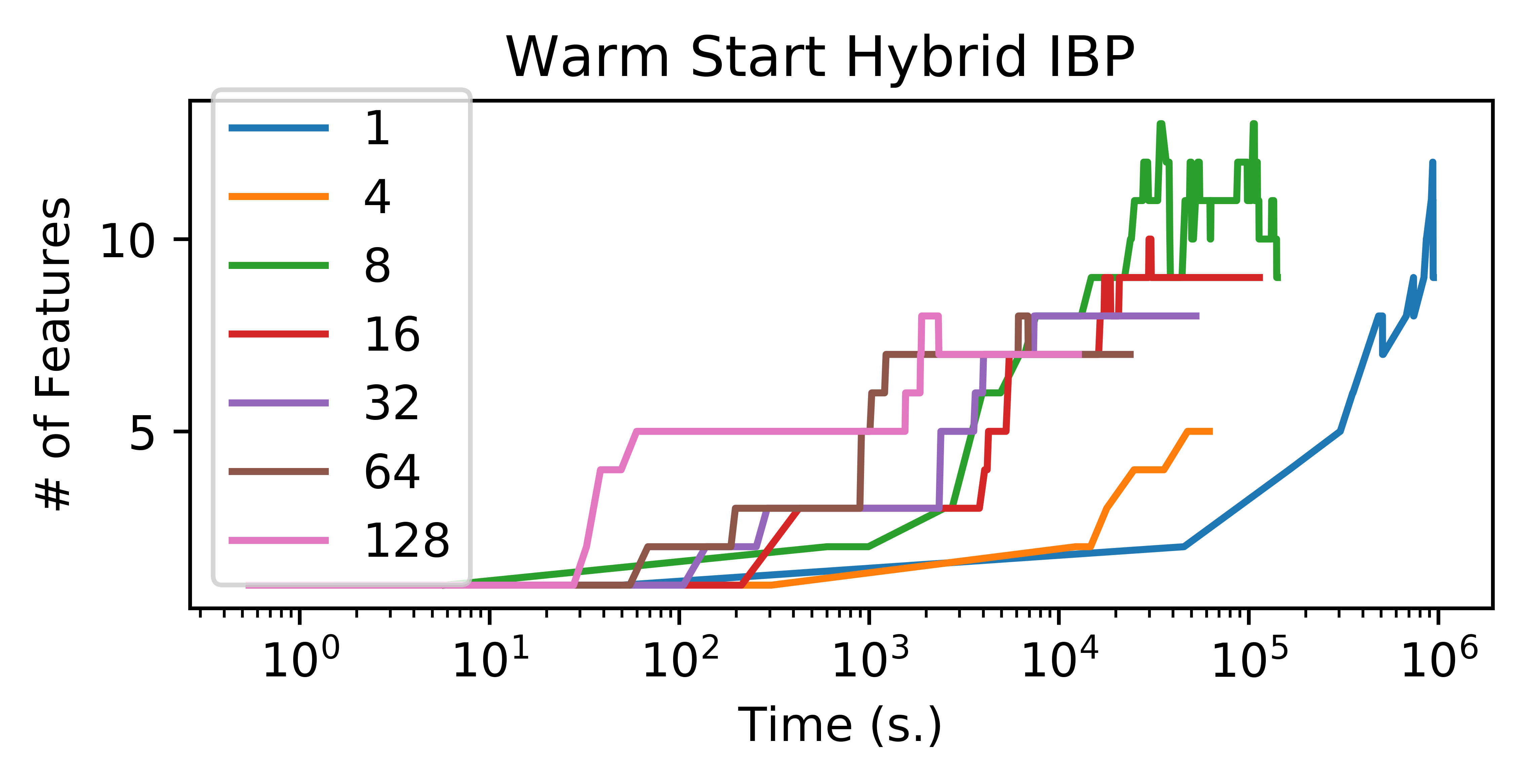}\label{fig:feat_warm}
\includegraphics[width=0.325\linewidth]{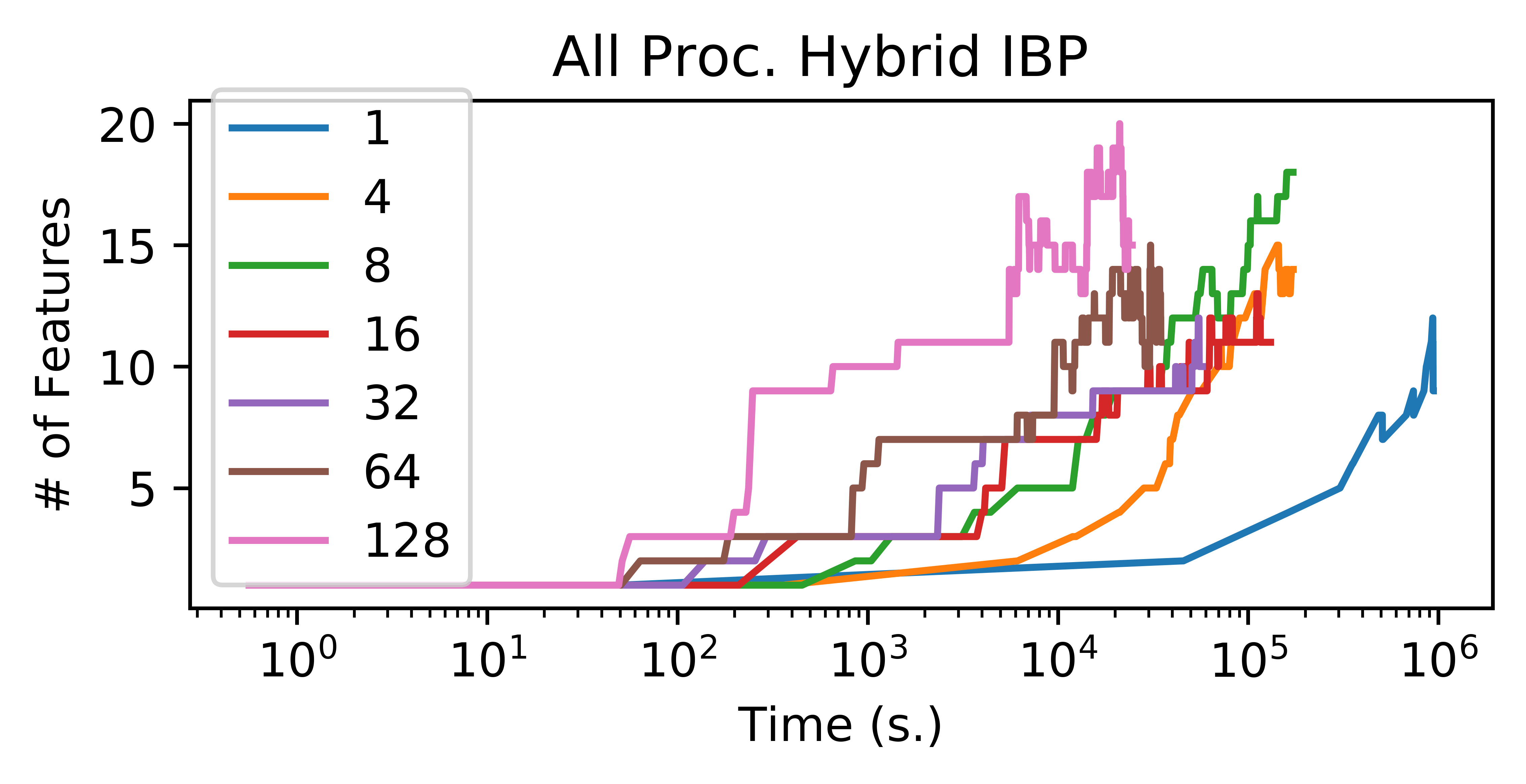}
\caption{\textbf{Top row left to right}: Test set log likelihood (y-axis) on synthetic data without warm-start initialization, with warm start, and with all processors on. The x-axis represents CPU wall time in seconds, on a log scale. 
\textbf{Bottom row left to right} number of features over iteration with cold-start, warm-start, and all processors introducing features. 
Y-axis represents number of instantiated features.}\label{fig:ibp}\vspace{-0.1in}
\end{figure*}
We evaluate the beta-Bernoulli sampler on synthetic data based on the ``Cambridge'' dataset, used in the original IBP paper \citep{Griffiths:Ghahramani:2011}, where each datapoint is the superposition of a randomly selected subset of four binary features of dimension 36, plus Gaussian noise with standard deviation 0.5.\footnote{See Figure~\ref{ibp_features} in the supplement for more details} We model this data using a linear Gaussian likelihood, with $Z\sim  \mbox{Beta-Bern}(\alpha, 1)$, $A_k \sim \mbox{Normal}(0,\sigma^2_A\mathbf{I})$, $X_n \sim \mbox{Normal}\left( \sum_k z_{nk}A_k, \sigma^2_X\mathbf{I} \right)$

We initialized to a single feature, and ran the hybrid sampler for 1,000 total observations with a synchronization step every $5$ iterations, distributing over 1, 4, 8, 16, 32, 64 and 128 processors. 

We first evaluate the hybrid algorithm under a ``cold start'', where only one processor is allowed to introduce new features for the entire duration of the sampler. In the top left plot of Figure~\ref{fig:ibp}, we see that the cold start results in slow convergence of in the test set log likelihood for large numbers of processors. 
 We can see in the bottom left plot of Figure~\ref{fig:ibp} that the number of features grows very slowly, as only $1/P$ processors are allowed to propose new features in the exact setting.



Next, we explore warm-start initialization, as described in Section~\ref{sec:warm}. For the first one-eighth of the total number of MCMC iterations, all processors can propose new features; after this we revert to the standard algorithm. 
The top central plot of Figure~\ref{fig:ibp} shows predictive log likelihood over time, and the bottom central plot shows number of features. We see that convergence is significantly improved relative to the cold-start experiments. Since we revert to the asymptotically correct sampler, the final number of features is generally close to the true number of features, 4.\footnote{Note that BNP models are not guaranteed to find the correct number of features in the posterior; see \citet{miller2013simple}} Additionally, we see that convergence rate increases monotonically in the number of processors. 


\begin{figure}
    \centering
    \includegraphics[width=.9\linewidth]{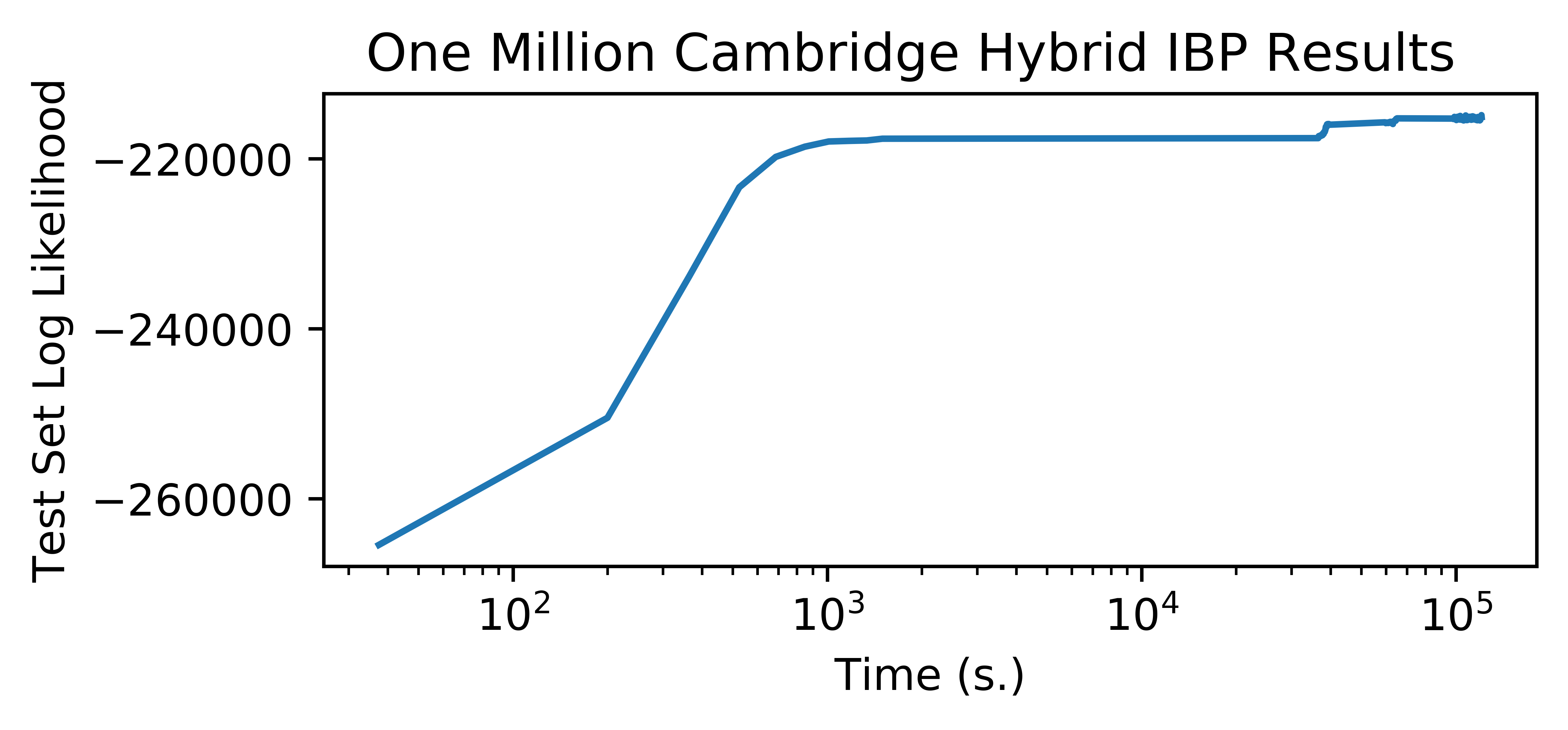}
    \caption{Test set log likelihood trace plot for a million observation ``Cambridge'' data set.}
    \label{fig:big_ibp}\vspace{-0.1in}
\end{figure}

Next, we allowed all processors to propose new features for the entire duration (``always-hot''). This setting approximately replicates the behavior of the parallel IBP sampler of \citet{Doshi-Velez:2009}. In the top right plot of Figure~\ref{fig:ibp}, we can see that all experiments roughly converge to the same test log likelihood. However, the number of features introduced (bottom right plot) is much greater than the warm start experiment, grows with the number of processors. Moreover, the difference in convergence rates between processors is not as dramatic as in the warm-start trials. 
 






Next, we demonstrate the scalability of our distributed algorithm on a massive synthetic example, showing it can be used for large-scale latent feature models. We generate one million ``Cambridge'' synthetic data points, as described for the previous experiments,  and distribute the data over 256 processors. This experiment represents the largest experiment ran for a beta-Bernoulli process algorithm (the next largest being 100,000 data points, in  \citealt{Doshi-Velez:2009}). We limited the sampler to run for one day and completed 860 MCMC iterations. In Figure~\ref{fig:big_ibp}, we see in the test set log likelihood traceplot that we can converge to a local mode fairly quickly under a massive distributed setting.

\subsubsection{Dirichlet Process}

Our distributed inference framework can also speed up inference in a DP mixture of Gaussians, using the version described in Algorithm~\ref{algo:distributed_dpmm}. We used a dataset containing the top 64 principle components of the CIFAR-100 dataset, as described in Section~\ref{sec:hybrid_experiments}. We compared against two existing distributed inference algorithms for the Dirichlet process mixture model, chosen to represent models based on both uncollapsed and collapsed samplers: 1) A DP variant of the asynchronous sampler of \cite{Smyth:Welling:Asuncion:2009}, an approximate collapsed method; and 2) the distributed slice sampler of \citet{Ge:Chen:Wan:Ghahramani:2015}, an uncollapsed method. 

Figure~\ref{fig:cifar100_methods} shows, when distributed over eight processors, our algorithm converges faster than the two comparison methods, showing that the high quality performance seen in Section~\ref{sec:hybrid_experiments} extends to the distributed setting. Further, in Figure~\ref{fig:cifar100_methods} we see roughly linear speed-up in convergence as we increase the number of processors from 1 to 8.

\begin{figure}[t]
\begin{center}
\includegraphics[width=.49\linewidth]{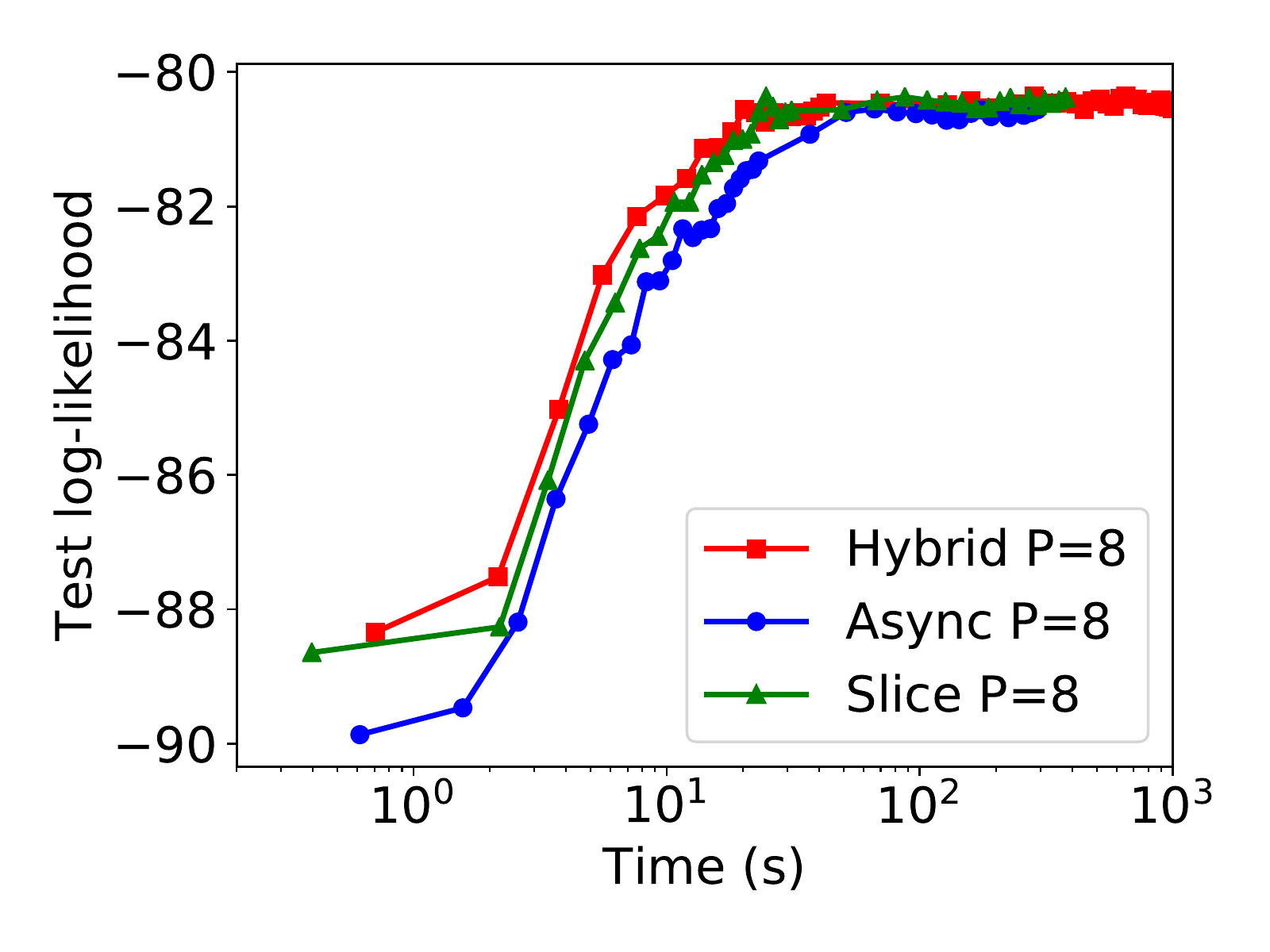}
\includegraphics[width=.49\linewidth]{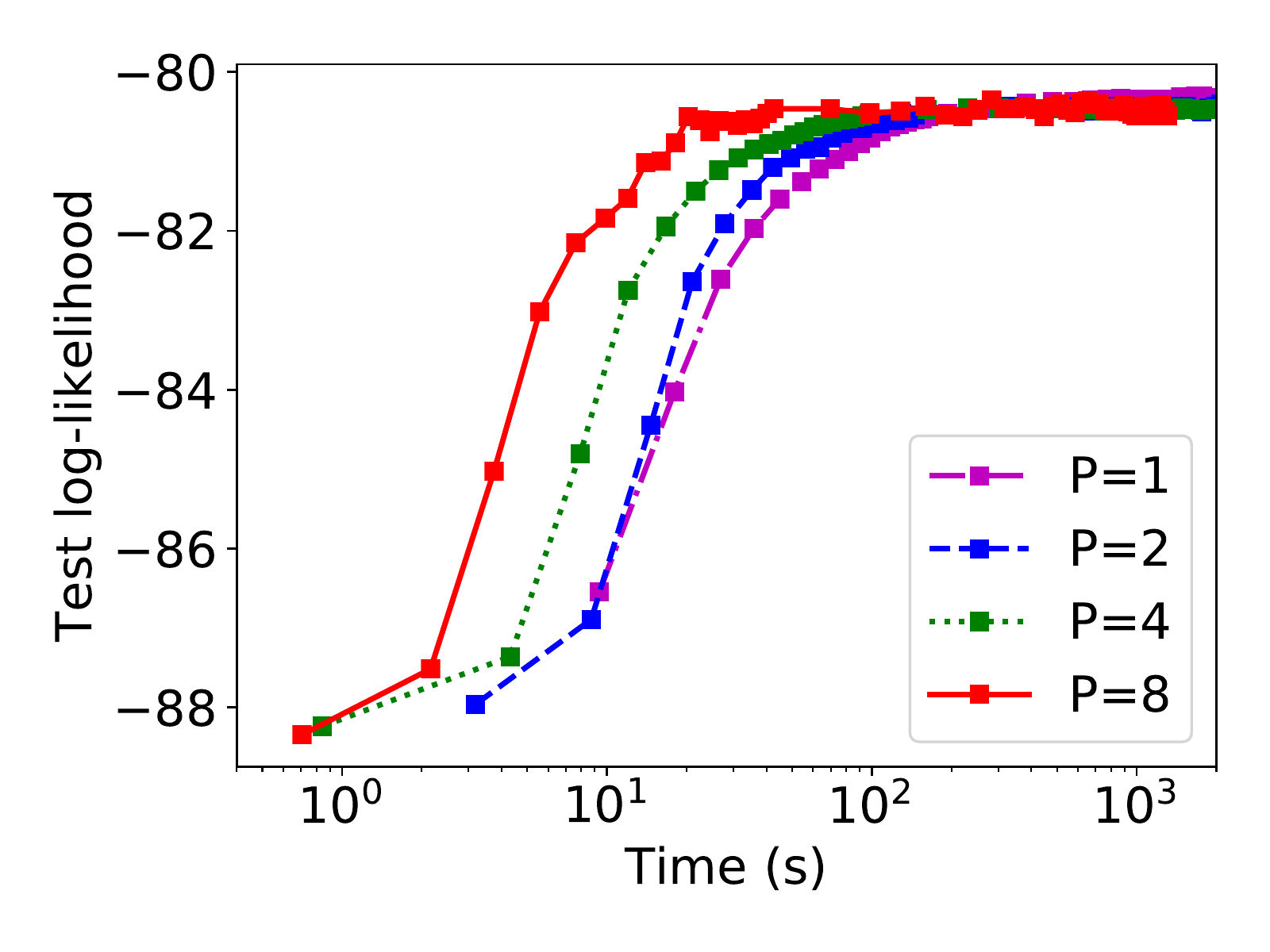}
\caption{Comparison of CIFAR-100 test log-likelihood with time, with baseline methods (\textbf{left}) and with varying number of processor (\textbf{right})
\label{fig:cifar100_methods}}
\end{center}
\end{figure}

\section{Conclusion}\label{discussion}

We have proposed a general inference framework for a wide variety of BNP models. We use the inherent decomposability of the underlying completely random measures to partition the latent random measures into a finite-dimensional component that represents the majority of the data, and an infinite-dimensional component that represents mostly uninstantiated tail. This allows us to take advantage of the inherent parallelizability of the uncollapsed sampler on the finite partition and the better performance of the collapsed sampler for proposing new components. Thus the proposed hybrid inference method can be easily distributed over multiple machines, providing provably correct inference for many  BNP models. Experiments show that, for both the DP and the beta-Bernoulli process, our proposed distributed hybrid sampler converges faster than the comparison methods. 


\bibliography{parallel}

\clearpage
\appendix
\onecolumn

\aistatstitle{Distributed, partially collapsed MCMC for Bayesian Nonparametrics: Supplement}


\section{Hybrid Sampler for the Pitman Yor Mixture Model (PYMM)\label{sec:pymm_supli}}
In this section, we expand upon Example 3 in Section~\ref{sec:hybrid}.

\textbf{Example 3: Pitman-Yor processes}
The Pitman-Yor process (\cite{perman1992size,pitman1997two}) is a distribution over probability measures, parametrized by a discount parameter $0\leq\sigma<1$, a concentration parameter $\alpha>-\sigma$, and a base measure $H$. While the Pitman-Yor process is not a normalized CRM, it can be derived from a $\sigma$‐stable CRM via a change of measure and normalization.  When the discount parameter is zero, we recover the Dirichlet process. As the discount parameter increases, we get increasingly heavy-tailed distributions over the atom sizes in the resulting probability measure.

\begin{lemma}
If $D \sim \mbox{PY}(\alpha, \sigma, H)$ with $\alpha>0$, and $Z_i \sim D$, then the posterior distribution $P(Y^*|Z_1,\dots, Z_n)$ is described by

\begin{equation}
    \begin{aligned}
    D^*_{\leq J} \sim& \mbox{PY}\left(\tilde{n} - J\sigma, 0, \frac{\sum_{k\leq j} (m_k-\sigma)\delta_{\theta_k}}{\tilde{n} -J\sigma}\right)\\  D^*_{>J}\sim& \mbox{PY}\Big(\alpha + n - \tilde{n} + J\sigma, \sigma, \\
    & \qquad \qquad \frac{(\alpha + Kd)H + \sum_{k>J}(m_k-\sigma)\delta_{\theta_k}}{\alpha+n-\tilde{n}+J\sigma}\Big)\\
    B \sim& \mbox{Beta}(\tilde{n}-J\sigma, \alpha + n - \tilde{n} +  J\sigma)\\
    D^* =& BD^*_{\leq J} +(1-B)D^*_{> J} 
    \end{aligned}
\end{equation}
where $K$ is the number of occupied clusters, $J\leq K$, and $\tilde{n} = \sum_{k=1}^Jm_k$.
\end{lemma}
\begin{proof}
The proof is a direct consequence of Lemma 22 in \cite{pitman1997two} and Theorem 1 in \citet{Dubey:Williamson:Xing:2014}. The special case for $J=K$ is presented in Corollary 20 of \citet{pitman1996}.
\end{proof}

We note that the posterior atom weights $(\pi_1,\dots, \pi_J)$ for the finite component $D^*_{\leq J}$ are distributed according to $\mbox{Dirichlet}(m_1-\sigma,\dots, m_J-\sigma)$, and can easily be sampled as part of an uncollapsed sampler.  Conditioned on $\{\pi_k, \theta_k: k\leq J\}$ and $B^*$ we can sample the cluster allocation, $Z_i$ of point $X_i$ as
\begin{equation}P(Z_i=k|-) \propto \begin{cases} B^*\pi_k f(x_n;\theta_k) & k\leq J\\
    \frac{(1-B^*)(m_k-\sigma)}{n-\tilde{n} + \alpha + J\sigma }f_k(x_n) & J < k \leq K\\
    \frac{(1-B^*)(\alpha + K \sigma)}{n-\tilde{n} + \alpha + J\sigma}f_H(x_n) & k=K+1
  \end{cases}
  \label{eqn:py_hybrid}
\end{equation}
where $f(X_i;\theta_k)$ is the likelihood for each mixing component; $f_k(X_i) = \int_\Theta f(X_i;\theta)p(\theta|\{X_j: Z_j=k, j\neq i\})d\theta$ is the conditional probability of $x_i$ given other members of the $k$th cluster; and $f_H(x_i) = \int_\Theta f(x_i;\theta) H(d\theta)$. This procedure is summarized in Algorithm~\ref{algo:hybrid_pymm}

\begin{algorithm}[h]
\caption{Hybrid PYMM Sampler\label{algo:hybrid_pymm}}
\begin{algorithmic}[1]
\While{not converged}
\State Select $J$
\State Sample $B^* \sim \mbox{Beta}(\tilde{n} - J\sigma, n-\tilde{n} + \alpha+J\sigma)$
\State Sample $(\pi_1,\dots,\pi_J) \sim \mbox{Dir}(m_1-\sigma, \ldots, m_J-\sigma)$
\State For $k\leq J$, sample 
$$ \theta_k \sim p(\theta_k|H, \{X_i:Z_i=k\})$$
\State For each data point $X_n$ sample $Z_i$ according to Equation \ref{eqn:py_hybrid}
\EndWhile
\end{algorithmic}
\end{algorithm}

We can similarly derive the distributed sampler for PYMM shown in algorithm \ref{algo:distributed_pymm}

\begin{algorithm}
\caption{Distributed PYMM Sampler\label{algo:distributed_pymm}}
\begin{algorithmic}[1]
\Procedure{Local}{$\{X_i, Z_i\}$} 

\Comment{Global variables $J, P^*,\{\theta_k, \pi_k\}_{k=1}^J, B^*$}
\If {$Processor = P*$}
\State Sample $Z_i$ according to \eqref{eqn:py_hybrid}
\Else
\State $P(Z_i = k) \propto \pi_k f(X_i;\theta_k)$
\EndIf
\EndProcedure

\Procedure{Global}{$\{X_i, Z_i\}$}
\State Gather cluster counts $m_k$ and parameter sufficient statistics $\Psi_k$ from all processors.
\State Let $J$ be the number of instantiated clusters.
\State Sample $B^* \sim \mbox{Beta}(n - J\sigma, \alpha+ J\sigma)$
\State Sample $(\pi_1, \dots, \pi_J) \sim \mbox{Dir}(m_1-\sigma, \ldots, m_J-\sigma)$
\State For $k: m_k>0$, sample 
$$ \theta_k \sim p(\theta_k|\Psi_k, H)$$
\State Sample $P^* \sim \mbox{Uniform}(1, \ldots, P)$
\EndProcedure
\end{algorithmic}
\end{algorithm}

\section{Hybrid Sampler for Hierarchical Dirichlet Processes\label{sec:hdp_supli}}
In this section, we expand upon Example 4 in Section~\ref{sec:hybrid}.

\textbf{Example 4: Hierarchical Dirichlet Process.} Hierarchical Dirichlet processes \citep[HDPs,][]{Teh:Jordan:Beal:Blei:2006} extend the DP to model grouped
data. The Hierarchical Dirichlet Process is a distribution over probability distributions $D_s,
s=1,\dots,S$, each of which is conditionally distributed according to
a DP\@. These distributions are coupled using a discrete common
base-measure $H$, itself distributed according to a DP\@. Each
distribution $D_s$ can be used to model a collection of observations
$\mathbf{x}_s := \{x_{si}\}_{i=1}^{N_s}$, where
\begin{equation}
\begin{array}{cc}
H\sim\mbox{DP}(\alpha, D_0),&
D_s\sim\mbox{DP}(\gamma,H),\\
\theta_{si} \sim D_s, &
x_{si}\sim f(\theta_{si}),
\end{array}\label{eq:HDP}
\end{equation}
for $s=1,\dots, S$ and $i=1,\dots, N_s$.

We consider a Chinese Restaurant Franchise \citep[CRF,][]{Teh:Jordan:Beal:Blei:2006} representation of the HDP, where each data point is represented by a customer, each atom in $D_s$ is represented by a table, and each atom location in the support of $H$ is represented by a dish. Let $x_{si}$ represent the $i$th customer in the $s$th restaurant; let $t_{si}$ be the table assignment of customer $x_{si}$; let $k_{st}$ be the dish assigned to table $t$ in restaurant $s$.  Let $m_{sk}$ denote the number of tables in restaurant $s$ serving disk $k$, and $n_{stk}$ denote the number of customers in restaurant $s$ at table $t$ having dish $k$. 

\begin{lemma}
Conditioned on the table/dish assignment counts $m_{\cdot k} = \sum_s m_{sk}$, the posterior distribution $P(H^*|\{t_{si}\}, \{k_{st}\}$ can be written as
$$
H^* = B^* H^*_{\leq J} + (1-B^*)H^*_{>J}$$
where
\begin{equation*}
    \begin{aligned}
 H^*_{\leq J}|m_{\cdot 1}, \dots, m_{\cdot K} \sim& \mbox{DP}\left(\tilde{m}, \frac{\sum_{k\leq J}m_{\cdot k}\delta_{\phi_k}}{\tilde{m}}\right)\\
    H^*_{> J}|m_{\cdot 1}, \dots, m_{\cdot K} \sim& \mbox{DP}\left(\alpha + m - \tilde{m}, \frac{\alpha H + \sum_{k> J}m_{\cdot k}\delta_{\phi_k}}{\alpha + m - \tilde{m}}\right)\\
     B^* \sim& \mbox{Beta}(\tilde{m}, m-\tilde{m}+\alpha), 
    \end{aligned}
\end{equation*}
where $K$ is the total number of instantiated dishes; $J\leq K$; $m=\sum_{k=1}^K m_{\cdot k}$; and $\tilde{m} = \sum_{k=1}^J m_{\cdot k}$
\end{lemma}
\begin{proof}
This is a direct extension of Lemma~\ref{lem:DP}, applied to the top-level Dirichlet process $H$.
\end{proof}

We can therefore construct a hybrid sampler, where $H^*_{\leq J}:=\sum_{k=1}^J \beta_k \delta_{\phi_k}$ is represented via $(\beta_1, \dots, \beta_J)\sim \mbox{Dir}(m_{\cdot 1}, \dots, m_{\cdot J})$ and corresponding $\phi_k \sim h(\phi_k) \prod_{s,i: k_{st_{si}} = k}f(x_{si} | \phi_k) $, and $H^*_{> J}$ is represented using a Chinese restaurant process. We can then sample the table allocations according to

\begin{equation}
\begin{aligned}
   t_{si} = t | t_{-si}, \{k_{st}\}, - \propto 
   &\begin{cases} 
   \frac{n_{st.}^{-si}}{  n_{s..}^{-si} + \gamma }f(x_{si};\phi_{k_{st}}) & \mbox{if } t \mbox{ previously used and } k_{st} \leq J  \\
   B^* \frac{\gamma}{ n_{s..}^{-si} + \gamma } \beta_{k_{st}} f(x_{si} ; \phi_{k_{st}}) &  t= \mbox{  new table table and }  k_{st} \leq J \\
   (1-B^*) \frac{\gamma}{ n_{s..}^{-si} + \gamma } \frac{m_{.k_{st}}}{\sum_{k: k > J}m_{.k} + \alpha} f(x_{si} |t_{-si},, - ;D_0) & \mbox{if } t \mbox{ new and } k_{st} > J \\ 
   (1-B^*)\frac{\gamma}{ n_{s..}^{-si} + \gamma } \frac{\alpha}{\sum_{k: k > J}m_{.k} + \alpha}  f(x_{si}; D_0)  & t \mbox{ new and } k = K + 1
  \end{cases}
  \end{aligned}\label{eqn:hdp_infinite}
\end{equation}

and sample each table according to

\begin{equation}
\begin{aligned}
    & p(k_{st} = k |\{t_{si}\}, k_{-st}) \propto
    &\begin{cases}
    B^*\beta_k p(\{ X_{si} : t_{si} = t\}|\phi_k) &k \leq J \\
    (1-B^*) \frac{m_{.k}}{\sum_{k>J} m_{.k} + \alpha} p(\{ X_{si} : t_{si} = t\}|D_0, \mathbf{X}) &J< k \leq K  \\
    (1-B^*) \frac{\alpha}{\sum_{k>J} m_{.k} + \alpha} p(\{ X_{si} : t_{si} = t\}|D_0) &  k = K + 1
    \end{cases}
\end{aligned}
\label{eqn:hdp_dish}
\end{equation}

We summarize the hybrid sampler shown in Algorithm~\ref{algo:hybrid_hdpmm}

\begin{algorithm}
\caption{Hybrid HDP Sampler\label{algo:hybrid_hdpmm}}
\begin{algorithmic}[1]
\While{not converged}
\State Select $J$
\State Sample $B^* \sim \mbox{Beta}(\tilde{m}, m-\tilde{m} + \alpha)$
\State Sample $(\beta_1,\dots,\beta_J) \sim \mbox{Dir}(m_{.1}, \ldots, m_{.J})$
\State For $k\leq J$, sample 
$$ \phi_k \sim p(\phi_k|D_0, \{X_{si}:k_{st_{si}}=k\})$$
\State For each data point $X_{si}$ sample $t_{si}$ according to Equation \ref{eqn:hdp_infinite}
\State Sample the dish $k_{st}$ according to Equation~\ref{eqn:hdp_dish}

\EndWhile
\end{algorithmic}
\end{algorithm}

If we ensure that data associated with each ``restaurant'' or group lies on the same table, we can extend this hybrid algorithm to a distributed setting, as described in Algorithm~\ref{algo:distributed_hdpmm}.

\begin{algorithm}
\caption{Distributed HDP Sampler\label{algo:distributed_hdpmm}}
\begin{algorithmic}[1]
\Procedure{Local}{$\{X_{si}, t_{si}, k_{st}\}$} 

\Comment{Global variables $J, P^*,\{\phi_k, \beta_k\}_{k=1}^J, B^*$}
\If {$Processor = P*$}
\State Sample $t_{si}$ according to \eqref{eqn:hdp_infinite}
\State Sample $k_{st}$ according to \eqref{eqn:hdp_dish}
\Else
\State Sample tables according to
\begin{equation}
    P(t_{si} = t | n_{st.}^{-si}) \propto
    \begin{cases}
    n_{st.}^{-si}f(x_{si} ; \phi_{k_{st}}) & \mbox{ if } t \mbox{ occupied} \\
    \gamma \beta_{k_{st}} f(x_{si} ; \phi_{k_{st}}) & \mbox{if } t \mbox{ is new}
    \end{cases}
\end{equation}
\State Sample dishes according to 
\begin{equation}
    P(k_{st} = k | \beta_k, \{t_{si}\}, X) \propto
    \beta_k p(\{ X_{si}: t_{si} = t \} | \phi_k )
\end{equation}
\EndIf
\EndProcedure

\Procedure{Global}{$\{X_{si}, t_{si}, k_{st}\}$}
\State Gather cluster counts $m_k$ and parameter sufficient statistics $\Psi_k$ from all processors.
\State Sample $B^* \sim \mbox{Beta}(\tilde{m} , \alpha + m - \tilde{m})$
\State Let $J$ be the number of instantiated clusters.
\State Sample $(\beta_1,\dots,\beta_J) \sim \mbox{Dir}(m_{.1}, \ldots, m_{.J})$
\State For $k\leq J$, sample 
$$ \phi_k \sim p(\phi_k|D_0, \{X_{si}:k_{st_{si}}=k\})$$
\State Sample $P^* \sim \mbox{Uniform}(1, \ldots, P)$
\EndProcedure
\end{algorithmic}
\end{algorithm}





\section{Further IBP Empirical Results}

For the ``Cambridge'' data sets described in the main paper, we generated images based on a superposition of the four features in the top row of Figure~\ref{ibp_features}, and then flattened the image to create a 36-dimensional vector. The bottom row of Figure~\ref{ibp_features} shows some sample data points.

In addition to synthetic data, we also evaluated the distributed beta-Bernoulli process sampler to a real-world data set, the CBCL MIT face dataset \citep{Weyrauch:Heisele:Huang:Blanz:2004}. This data set consists of 2,429 images of $19 \times 19$ dimensional faces. We distributed the data across 32 processors and ran the sampler in parallel for $500$ iterations. Figure~\ref{fig:face_LL} shows the test set log likelihood of the sampler over time and Figure~\ref{fig:face_features} shows the features learned by the hybrid IBP method and we can clearly see that our method can discover the underlying facial features in this data set.
\begin{figure}[hbtp]
\begin{center}
\includegraphics[scale=.5]{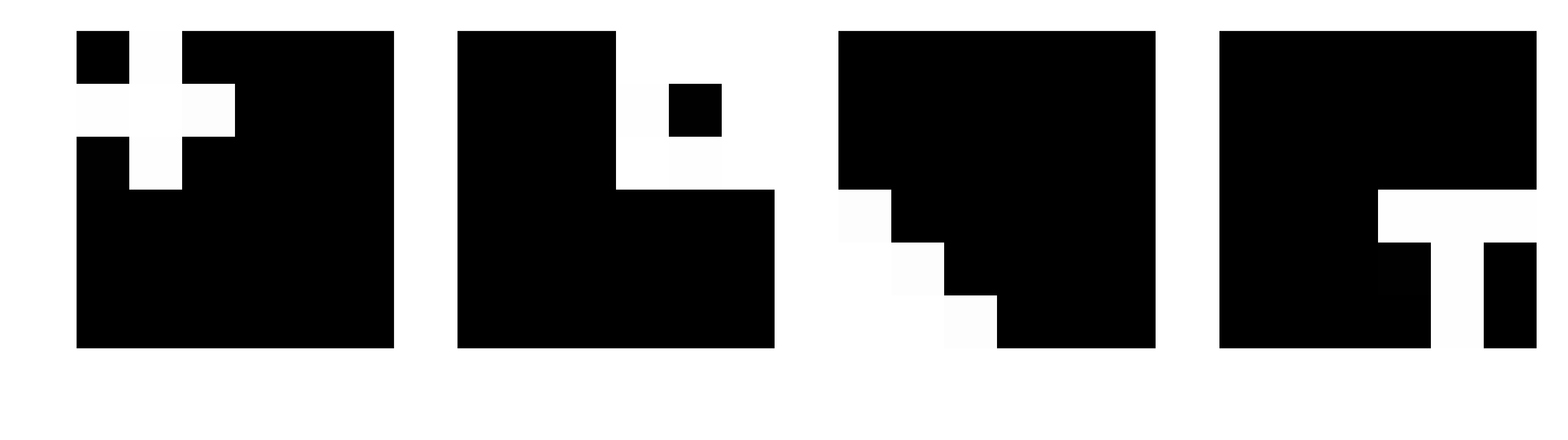}\\
\includegraphics[scale=.5]{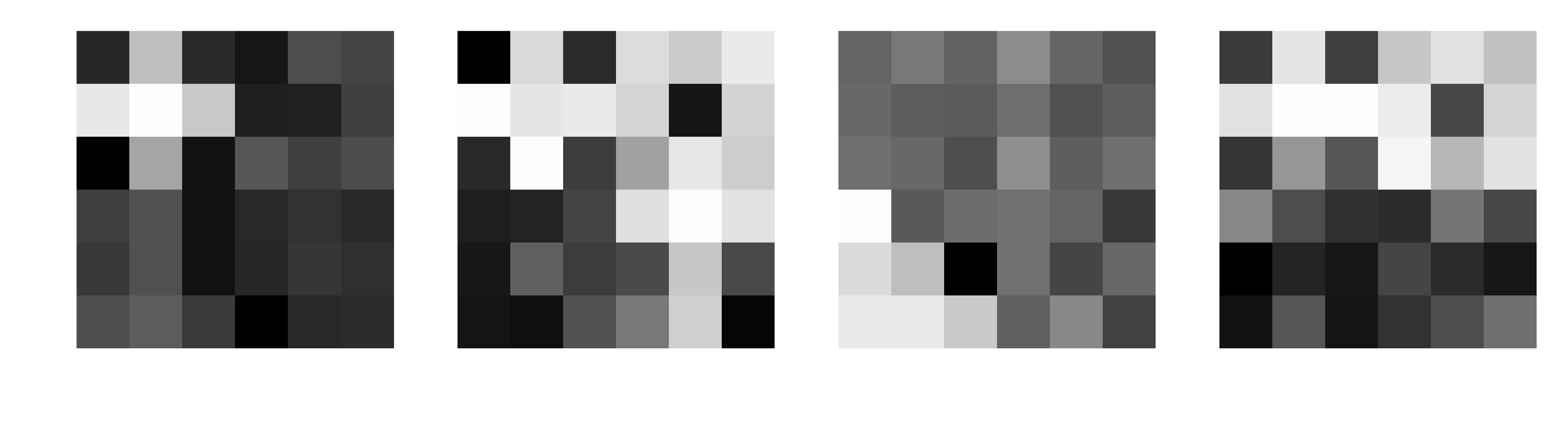}\\
\end{center}
\caption{Top: The true features present in the synthetic data set. Bottom: Examples of observations in the synthetic data set. }
\label{ibp_features}
\end{figure}

\begin{figure}
    \centering
    \includegraphics[width=.5\linewidth]{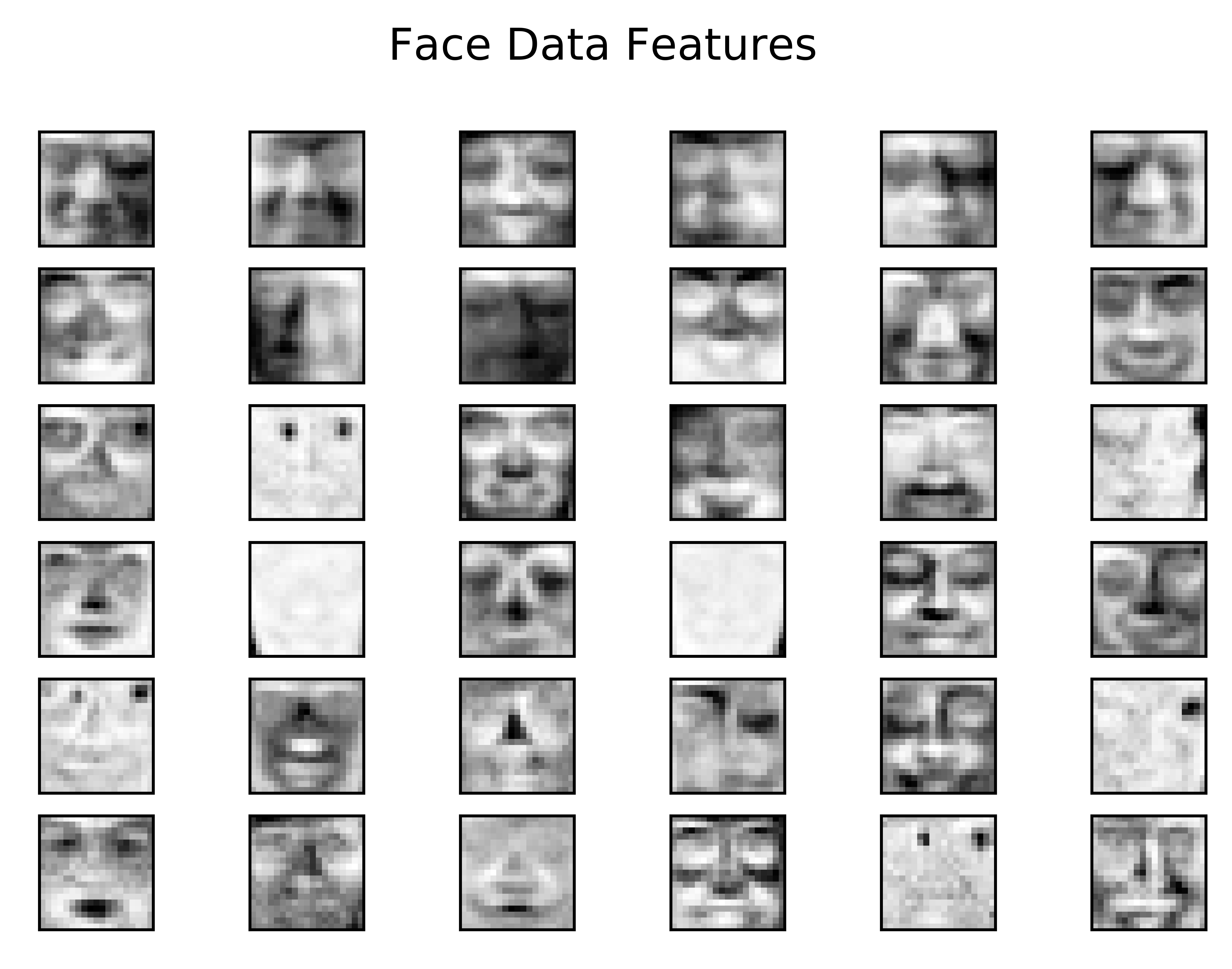}
    \caption{Learned features from the hybrid IBP for the CBCL face data set.}
    \label{fig:face_features}
\end{figure}

\begin{figure}
    \centering
    \includegraphics[width=.5\linewidth]{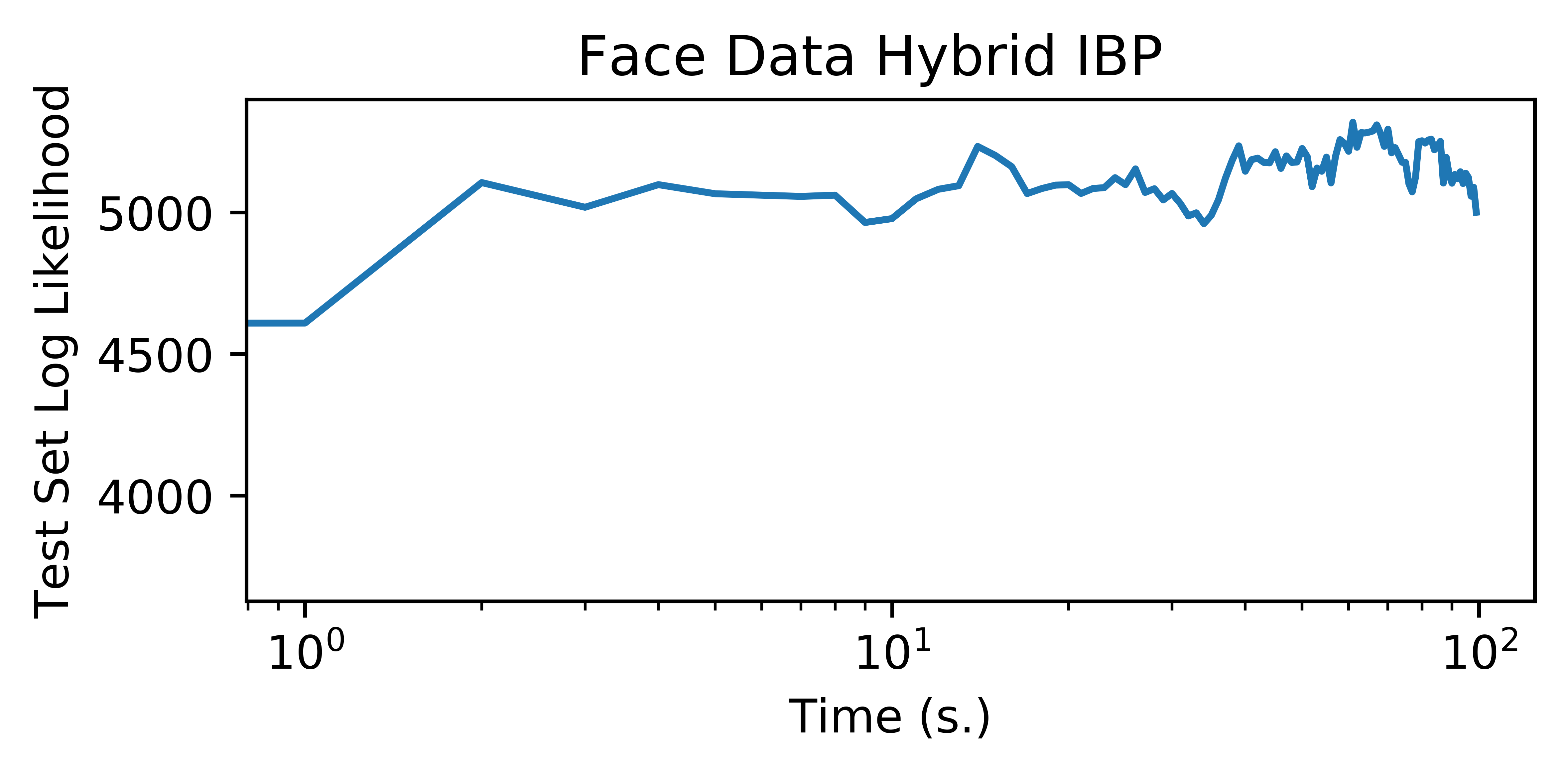}
    \caption{Test set log-likelihood trace plot for CBCL face data set.}\label{fig:face_LL}
\end{figure}

\end{document}